\documentclass[oneside]{article}

\usepackage[utf8]{inputenc}
\usepackage[T1]{fontenc}
\usepackage{hyperref}
\usepackage{url}
\usepackage{nicefrac}
\usepackage{microtype}
\usepackage{xcolor}
\usepackage{amsfonts, amsmath, mathtools, amssymb, amsthm}
\usepackage{booktabs, multirow}
\usepackage{algorithm}
\usepackage[noend]{algorithmic}
\usepackage{graphicx}
\usepackage[numbers]{natbib}
\usepackage{caption}

\usepackage[
    letterpaper,
    left=0.9in,
    right=0.9in,
    top=1in,
    bottom=1in
]{geometry}

\input{macros.sty}

\title{Bring Your Own Algorithm for Optimal Differentially Private Stochastic Minimax Optimization}

\author{
    Liang Zhang \thanks{Department of Computer Science, ETH Zurich and Max Planck ETH Center for Learning Systems. \texttt{liang.zhang@inf.ethz.ch}}\\
    \and
    Kiran Koshy Thekumparampil \thanks{Amazon Search, Palo Alto, and Department of Electrical and Computer Engineering, University of Illinois at Urbana-Champaign. \texttt{thekump2@illinois.edu}} \\
    \and
    Sewoong Oh \thanks{Paul G. Allen School of Computer Science and Engineering, University of Washington. \texttt{sewoong@cs.washington.edu}} \\
    \and
    Niao He \thanks{Department of Computer Science, ETH Zurich. \texttt{niao.he@inf.ethz.ch}} \\
}

\begin{document}

\maketitle

\begin{abstract}
    We study \emph{differentially private} (DP) algorithms for smooth stochastic minimax optimization, with stochastic minimization as a byproduct. The holy grail of these settings is to guarantee the optimal trade-off between the privacy and the excess population loss, using an algorithm with a \emph{linear time-complexity} in the number of training samples. We provide a general framework for solving differentially private stochastic minimax optimization (DP-SMO) problems, which enables the practitioners to bring their own base optimization algorithm and use it as a black-box to obtain the near-optimal privacy-loss trade-off. Our framework is inspired from the recently proposed Phased-ERM method \citep{feldman2020private} for nonsmooth differentially private stochastic convex optimization (DP-SCO), which exploits the stability of the empirical risk minimization (ERM) for the privacy guarantee. The flexibility of our approach enables us to sidestep the requirement that the base algorithm needs to have bounded sensitivity, and allows the use of sophisticated variance-reduced accelerated methods to achieve near-linear time-complexity. To the best of our knowledge, these are the first near-linear time algorithms with near-optimal guarantees on the population duality gap for smooth DP-SMO, when the objective is (strongly-)convex--(strongly-)concave. Additionally, based on our flexible framework, we enrich the family of near-linear time algorithms for smooth DP-SCO with the near-optimal privacy-loss trade-off.
\end{abstract}

\section{Introduction}

Machine learning models are nowadays trained using large corpora of data samples collected from many different entities, e.g.,~from users of a large software service \cite{kairouz2021advances}. However, it has been empirically shown that these models can be exploited to reveal private information about these contributing entities. For example, \citet{carlini2021extracting} attacked the large language model, GPT-2, to reveal hundreds of verbatim text samples used to train these models. These attacks violate the privacy of the contributing entities, but naturally they expect that no private information about them can be revealed through these models. Over the last decade, this expectation was even legislated into laws such as GDPR in EU \cite{cummings2018role}.
There are many mathematical frameworks formalizing this expectation of privacy, but the most widely accepted one is that of \emph{Differential Privacy} (DP) \citep{dwork2006calibrating}. With high probability, models satisfying DP cannot be attacked by any adversary to identify that a particular training sample was used in its training. Hence, DP provides any entity plausible deniability that they contributed to the training set. However, optimization algorithms for training such models under DP require careful design choices to ensure privacy while preventing the degradation of convergence speed and data efficiency. This has led to the burgeoning field of differentially private optimization algorithms \citep{bassily2014private, bassily2019private, feldman2020private}, which considers stochastic convex minimization as the canonical problem.

For differentially private stochastic convex optimization (DP-SCO) with $(\varepsilon,\delta)$-DP guarantees, the optimal excess population risk is $\Theta(1/(\mu n) + d\log(1/\delta)/(\mu n^2\varepsilon^2))$ for $\mu$-strongly convex objectives, where $n$ is the number of participants and $d$ is the dimension of the variable. If the objective is also smooth, this can be achieved with a linear-time $(\varepsilon,\delta)$-DP algorithm \citep{feldman2020private} using $\cO(n)$ stochastic gradient evaluations. This analysis critically relies on the concept of algorithmic {\em stability} \citep{bousquet2002stability}, which measures how much the population loss of an algorithm's output changes when a single data point in the input dataset is perturbed. This is also known as {\em sensitivity} of an algorithm in DP, which determines how much noise needs to be added in order to achieve $(\varepsilon,\delta)$-DP (see Definition \ref{def:gauss}). We will use stability and sensitivity interchangeably when referring to optimization algorithms. 
As the sensitivity of stochastic gradient descent (SGD) is known \citep{hardt2016train}, \citet{feldman2020private} were able to add an appropriate amount of noise---tailored to this sensitivity---to an SGD-based algorithm to achieve optimal risk in linear time. However, such tight analysis of sensitivity is generally intractable for more complex optimization routines that practitioners might want to use. Further, the stability analysis of SGD only holds when the smoothness parameter $\ell$ of the problem is upper-bounded by $\tilde\cO(\mu n)$. In the first part of the paper, we show that we can achieve similar guarantees using a wider choice of, potentially more practical, ``base'' algorithms, without any restrictions on the smoothness parameter.

On another front, many emerging practical machine learning applications are formulated as stochastic minimax optimization problems, e.g., generative adversarial networks \citep{goodfellow2014generative}, adversarially robust machine learning \citep{madry2018towards}, and reinforcement learning \citep{dai2018sbeed}. Designing DP algorithms for solving these minimax problems is of paramount importance. For example, private generative adversarial networks provide a promising new direction to synthetic data generation \citep{xie2018differentially}, such as in networked time-series data \citep{lin2020using}. 
Motivated by such applications, we study the differentially private stochastic minimax optimization (DP-SMO) problem of the form: 
\begin{equation*}
    \min_{x\in\cX} \max_{y\in\cY}\; F(x,y) \triangleq \bE_\xi[f(x,y;\xi)],
\end{equation*}
where the objective $f(x,y;\xi)$ is smooth and convex-concave for any random vector $\xi$, and we are given access to $n$ i.i.d. samples $\{\xi_i\}_{i=1}^n$. In contrast to DP-SCO, where linear-time algorithms have been proposed to achieve optimal risk guarantees, existing private algorithms \citep{boob2021optimal, yang2022differentially} for DP-SMO achieving optimal guarantees have time-complexity which scales super-linearly in the number of samples $n$ (summarized in Table~\ref{tab:compare}).  
In the second part of this paper, we close this gap by introducing a new class of output perturbation-based DP algorithms for both strongly-convex--strongly-concave and convex-concave settings, which can achieve near-optimal population risk bounds using $\tilde\cO(n)$ stochastic gradient computations, where $\tilde\cO$ hides logarithmic factors.

One of the main bottlenecks for (i) using a wider variety of algorithms to solve the DP-SCO problem, and (ii) solving the DP-SMO problem using linear-time algorithms, is the lack of known stability results for fast and sophisticated non-private algorithms, such as variance-reduced and accelerated methods. One can sidestep the algorithm-specific stability requirement by utilizing the stability of the optimal solution to a strongly-convex and Lipschitz empirical risk minimization (ERM) problem \cite{shalev2009stochastic, zhang2021generalization}. Particularly, if the output of an algorithm for such an ERM problem is close enough to its empirical solution, then the sensitivity of the algorithm is automatically guaranteed. Exploiting this observation, \citet{feldman2020private} proposed the phased-ERM algorithm for nonsmooth DP-SCO, which solves a series of strongly-convex ERM subproblems to sidestep the stability analysis of SGD for nonsmooth functions. This algorithm achieves a quadratic time-complexity for nonsmooth DP-SCO.

We observe that when the problem is additionally smooth, or when it allows a smooth minimax reformulation \cite{nesterov2005smooth}, there exist a plethora of fast algorithms to solve the resulting ERM subproblems. Combining these algorithms with (phased) output perturbation gives rise to a class of near-optimal near-linear \footnote{We only claim near-optimality of our rate because of its dependence on the condition number or additional logarithmic terms. We call an algorithm near-linear time when it is linear-time up to some logarithmic factors.} time-complexity private algorithms for both smooth DP-SCO and DP-SMO problems, without any additional effort on their stability analysis. Our contributions are summarized below:

$\bullet$ We introduce a flexible framework for solving smooth DP-SMO and smooth DP-SCO problems, utilizing the (phased) output perturbation mechanism. The black-box framework enables us to bypass the need to prove algorithm-specific stability and transform off-the-shelf optimization algorithms into DP algorithms with near-optimal guarantees.
This is attractive as there are currently no stability analyses for accelerated \citep{allen2017katyusha, palaniappan2016stochastic} and variance-reduced \citep{johnson2013accelerating, palaniappan2016stochastic} algorithms for both SMO and SCO.

$\bullet$ Using the framework, we provide the first near-linear time private algorithms for smooth DP-SMO with near-optimal bound on the population duality gap, under both strongly-convex--strongly-concave and convex-concave cases (see Table \ref{tab:compare}).
Among other things, this implies that if a (primal) {\em nonsmooth minimization} problem can be reformulated as a smooth convex-concave minimax optimization problem, then we can solve it in near-linear time instead of the best known super-linear time \citep{asi2021private, kulkarni2021private}. In prior work, near-linear time DP algorithms with optimal rates for solving nonsmooth convex objectives only existed for generalized linear losses \citep{bassily2021differentially}.

$\bullet$ The framework also enriches the cohort of near-linear time near-optimal private algorithms for smooth DP-SCO settings, which only contained SGD previously. Moreover, existing optimal DP algorithms for smooth DP-SCO rely on a stability result of SGD, which only holds for a restricted range of the smoothness parameter \citep{hardt2016train}. Such restrictions are avoided in our framework.

\renewcommand{\arraystretch}{2}
\captionsetup[table]{skip=5pt}

\begin{table}
    \centering
    \caption{
    Among $(\varepsilon,\delta)$-DP smooth minimax algorithms that achieve the near-optimal utility bound on the population strong duality gap for the $\mu$ strongly-convex--strongly-concave case and the population weak duality gap for the convex-concave case, the proposed framework achieves the best gradient complexity. Here $\tilde\cO$ hides logarithmic terms, $\ell$ is the smoothness parameter, $\kappa=\ell/\mu$ and $d=\max\{d_x,d_y\}$. Lower-bounds are also summarized in the table.
    }
    \begin{tabular}{ccccc}
        \toprule
        Settings & Lower-bound & Algorithm & Utility & Complexity \\
        \midrule
        SC-SC
        & $\Omega\roundBr[\Big]{\frac{1}{\mu_x n} + \frac{d_x\log(1/\delta)}{\mu_x n^2\varepsilon^2}}$
        & Thm.~\ref{thm:scsc}
        & $\cO\roundBr[\Big]{\frac{\kappa^2}{\mu n} + \frac{\kappa^2d\log(1/\delta)}{\mu n^2\varepsilon^2}}$
        & $\tilde\cO(n + \sqrt{n}\kappa)$
        \\
        \midrule
        \multirow{4}{2em}{\\[-0.9em] C-C}
        & \multirow{4}{9.5em}{$\Omega\roundBr[\Big]{\frac{1}{\sqrt{n}} + \frac{\sqrt{d_x\log(1/\delta)}}{n\varepsilon}}$}
        & DP-SGDA \citep{yang2022differentially}
        & $\cO\roundBr[\Big]{\frac{1}{\sqrt{n}} + \frac{\sqrt{d\log(1/\delta)}}{n\varepsilon}}$
        & $\cO(n^{3/2}\sqrt{\varepsilon})$
        \\
        & & NSEG \citep{boob2021optimal}
        & $\cO\roundBr[\Big]{\frac{1}{\sqrt{n}} + \frac{\sqrt{d\log(1/\delta)}}{n\varepsilon}}$
        & $\cO(n^2)$
        \\
        & & NISPP \citep{boob2021optimal}
        & $\cO\roundBr[\Big]{\frac{1}{\sqrt{n}} + \frac{\sqrt{d\log(1/\delta)}}{n\varepsilon}}$
        & $\tilde\cO(n^{3/2})$
        \\
        & & Thm.~\ref{thm:cc}
        & $\tilde\cO\roundBr[\Big]{\frac{1}{\sqrt{n}} + \frac{\sqrt{d\log(1/\delta)}}{n\varepsilon}}$
        & $\tilde\cO(n)$ \\
        \bottomrule
    \end{tabular}
    \label{tab:compare}
\end{table}

\paragraph{Related Works:} Differentially private optimization has been an active research field for the past few years, and early works focused on the \emph{empirical} problems \citep{bassily2014private, wu2017bolt, zhang2017efficient, wang2017differentially}. In addition to the output perturbation used in this paper, many existing works applied gradient perturbation to guarantee privacy. This method adds noise to each iteration of the algorithm and then applies moments accountant \citep{abadi2016deep} or advanced composition \citep{kairouz2015composition} to analyze the overall privacy. Although gradient perturbation does not need the smoothness or convexity assumptions of the function and works for most iterative algorithms, it requires larger mini-batches \citep{bassily2019private} or longer training time \citep{bassily2020stability} if one resorts to the privacy amplification via subsampling \citep{beimel2014bounds} to reduce the DP noise, resulting in super-linear gradient queries in existing methods \citep{bassily2019private, yang2022differentially}. Here we mainly review previous results that achieve the optimal utility bound $\cO(1/\sqrt{n}+\sqrt{d\log(1/\delta)}/(n\varepsilon))$ on the \emph{population} loss---according to a lower-bound in \citet{bassily2014private}---when solving a $d$ dimensional DP-SCO problem using $n$ samples with $(\varepsilon, \delta)$-DP guarantees.

In the smooth convex case, \citet{bassily2019private} is the first to derive optimal rates for DP-SCO with complexity $\cO(n^{3/2}\sqrt{\varepsilon})$ by gradient perturbation and stability of SGD \citep{hardt2016train}. \citet{feldman2020private} provided two SGD-based linear-time algorithms: one uses privacy amplification by iteration \citep{feldman2018privacy} that only works for contractive updates; the second one uses phased output perturbation and stability of SGD. In the nonsmooth convex case, \citet{bassily2020stability} established the stability of SGD for nonsmooth functions and obtained a quadratic-time algorithm. \citet{feldman2020private} leveraged the stability of ERM solutions \citep{shalev2009stochastic} and achieved $\cO(n^2\log(1/\delta))$ complexity with phased output perturbation. Based on this phased framework, \citet{asi2021private} and \citet{kulkarni2021private} used gradient perturbation and improved the complexity to $\cO(\min\{n^{3/2}, n^2/\sqrt{d}\})$ and $\cO(\min\{n^{5/4}d^{1/8}, n^{3/2}/d^{1/8}\})$ respectively; \citet{asi2021stochastic} introduced a hypothetical linear-time algorithm assuming the existence of a low-biased estimator. Moreover, \citet{bassily2021differentially} gave a near-linear time algorithm for nonsmooth convex generalized linear losses using phased SGD \citep{feldman2020private} with smoothing techniques \citep{nesterov2005smooth}.

To the best of our knowledge, only few papers studied DP-SMO and all of them used gradient perturbation to guarantee privacy. \citet{boob2021optimal} analyzed stability of Extragradient \citep{tseng1995linear} and proximal point methods \citep{rockafellar1976monotone} for smooth convex-concave functions, and their DP versions run with time $\cO(n^2)$ and $\tilde\cO(n^{3/2})$ respectively. \citet{yang2022differentially} used the stability of SGDA \citep{lei2021stability} and obtained DP algorithms with complexity $\cO(n^{3/2}\sqrt{\varepsilon})$ for the smooth convex-concave case and $\cO(n^2)$ for the nonsmooth case, mirroring the guarantees of SGD for DP-SCO \citep{bassily2019private, bassily2020stability}. \citet{kang2022stability} only focused on stability and generalization analysis of gradient perturbed SGDA and provided high-probability results. See Table \ref{tab:compare} for a brief comparison.

The idea of using the stability of ERM for smooth convex optimization has also been exploited before. \citet{attia2022uniform} developed a black-box framework for smooth convex objectives that produces uniformly-stable algorithms while maintaining fast convergence rates. \citet{lowy2021output} considered DP-SCO by output perturbation but only provided near-linear time near-optimal algorithms for smooth strongly-convex losses. For smooth convex case, they achieved a sub-optimal rate $\cO(1/\sqrt{n} + (\sqrt{d\log(1/\delta)}/(n\varepsilon))^{2/3})$ in near-linear time. In the setting of smooth strongly-convex--(strongly-)concave DP-SMO, they directly utilized its DP-SCO reformulation in the primal function to obtain the final guarantees. Therefore, their utility bound is simply on the primal risk, which is weaker than the primal-dual gap considered in this work. The algorithms run in near-linear time for the strongly-concave case and super-linear time $\cO(n^{5/2})$ for the concave case. In contrast, we provide near-optimal near-linear time algorithms in all aforementioned settings.

\paragraph{Notations:}
We use $\norm{\cdot}$ for the Euclidean norm of a vector and $\abs{\cdot}$ for the absolute value or the cardinality of a set.
A function $g: \bR^d\rightarrow\bR$ is $L$-Lipschitz if $\abs{g(x_1)-g(x_2)}\leq L\norm{x_1-x_2}$ for $x_1, x_2$ in the domain of $g$.
A function $h: \bR^d\rightarrow\bR$ is $\ell$-smooth if it is differentiable and $h(x_2)\leq h(x_1)+\nabla h(x_1)^\top(x_2-x_1)+(\ell/2)\norm{x_1-x_2}^2$.
A function $p: \bR^d\rightarrow\bR$ is convex if $p(\alpha x_1 + (1-\alpha) x_2) \leq \alpha p(x_1)+ (1-\alpha) p(x_2)$ for all $\alpha\in[0,1]$, and $p$ is $\mu$-strongly convex if $p(x)-(\mu/2)\norm{x}^2$ is convex with $\mu>0$.
A function $q: \bR^d\rightarrow\bR$ is concave if $-q$ is convex and $\mu$-strongly concave if $-q$ is $\mu$-strongly convex.
For a vector $x\in\bR^d$, the notation $x+\cN(0, \sigma^2\rI_d)$ means $x+z$ for a random vector $z\sim\cN(0,\sigma^2\rI_d)$ sampled from the Gaussian distribution.

\section{Preliminaries}

We first provide some background on differential privacy and stochastic minimax optimization.

\subsection{Differential Privacy}

Differential Privacy (DP), introduced in \citet{dwork2006calibrating}, measures privacy leakage of an algorithm.

\begin{definition}
    For two datasets $S =\{\xi_i\}_{i=1}^n$ and $S'=\{\xi_i'\} _{i=1}^{n}$, we say the pair $(S,S')$ is \emph{neighboring} if
    $\max\{| S \setminus S'|,| S' \setminus S|\} = 1$
    and we denote neighboring datasets with $S \sim S'$. 
    For an algorithm $\cA$ and some privacy parameters $\varepsilon>0$ and $\delta\in(0,1)$, we say $\cA$ satisfies
    $(\varepsilon,\delta)$-\emph{differential privacy} if
    $\bP(\cA(S)\in A) \leq e^\varepsilon \bP(\cA(S') \in A) + \delta$
    for all $S\sim S'$ and all subset $A$ in the range of $\cA$.
    \label{def:dp}
\end{definition}

In this work, we focus on the settings when $\varepsilon\in(0,1)$ and $\delta\in(0,1/n)$ given dataset of size $n$. Sensitivity is an important concept that makes designing $(\varepsilon,\delta)$-DP mechanisms straightforward.

\begin{definition} 
    \label{def:gauss}
    Let $\cA$ be some randomized algorithm operating on $S$ and outputting a vector in $\bR^d$. If $\cA$ has {\em sensitivity} $\Delta_\cA :=\sup_{S \sim S'} \norm{\cA(S) - \cA(S')}$ with probability at least $1-\delta$, then the \emph{Gaussian mechanism} outputs
    $\cA(S)+\cN(0,(\Delta_\cA \sqrt{2\log(1.25/\delta)}/\varepsilon)^2 \rI_d)$
    and achieves $(\varepsilon, 2\delta)$-DP \citep{dwork2014algorithmic, feldman2020private}.
\end{definition}

The following basic composition rule of differential privacy will be used in the analysis.

\begin{lemma}
    If $\cA_1$ is $(\varepsilon_1, \delta_1)$-DP and  $\cA_2$ is $(\varepsilon_2,\delta_2)$-DP, then $(\cA_1, \cA_2)$ is $(\varepsilon_1+\varepsilon_2, \delta_1+\delta_2)$-DP \citep{dwork2014algorithmic}. 
    For a sequence of interactive algorithms 
    $\{\cA_k\}_{k=1}^K$ each satisfying  $(\varepsilon_k, \delta_k)$-DP and operating on a subset $S_k$, if  $S_k$'s are disjoint then the composition $(\cA_1(S_1), \cA_2(S_2), \ldots, \cA_K(S_K))$ is
    $(\max_{k\in[K]} \varepsilon_k, \max_{k\in[K]}\delta_k)$-DP  
    (known as parallel composition in \citet{mcsherry2009privacy}).
    \label{lm:composition}
\end{lemma}

\subsection{Stochastic Minimax Optimization}

The stochastic minimax (a.k.a.~saddle point) optimization problem has the form:
\begin{equation} \label{eq:smo_formulation}
    \min_{x\in\cX} \max_{y\in\cY}\; F(x,y) \triangleq \bE_{\xi}[f(x,y;\xi)],
\end{equation}
where $F$ is the population-level expectation of the stochastic continuous objective $f(\cdot, \cdot; \xi): \bR^{d_x}\times\bR^{d_y}\rightarrow\bR$ with closed convex domains $\cX\subset\bR^{d_x}$ and $\cY\subset\bR^{d_y}$ whose stochasticity is captured by the random vector $\xi$. We are interested in the \emph{population saddle point} $(x^*, y^*)\in\cX\times\cY$ of the above problem such that $F(x^*, y) \leq F(x^*, y^*) \leq F(x, y^*)$ for all $(x,y)\in\cX\times\cY$.

For some randomized algorithm with output $(\tilde x, \tilde y)$, we measure its convergence rates by the \emph{population strong duality gap} $\bE[\max_{y\in\cY} F(\tilde x, y) - \min_{x\in\cX} F(x, \tilde y)]$ or the \emph{population weak duality gap} $\max_{y\in\cY}\bE[F(\tilde x, y)] - \min_{x\in\cX}\bE[F(x, \tilde y)]$. Note that both duality gaps are always larger than $0$, and a deterministic $(\tilde x, \tilde y)$ is the saddle point if and only if the duality gaps are $0$.

In practice, we usually do not have access to the distribution $P_\xi$ of the random vector $\xi$. Instead we are given a dataset $S=\{\xi_i\}_{i=1}^n$ with $n$ random vectors independently sampled from the distribution $P_\xi$. We define the \emph{empirical} minimax optimization problem as
\begin{equation}
    \min_{x\in\cX} \max_{y\in\cY} \;  \hat F_S(x,y) \triangleq
    \frac{1}{n}\sum_{i=1}^n f(x,y;\xi_i).
\end{equation}
Similarly, we can define the \emph{empirical saddle point} $(\hat x_S^*, \hat y_S^*)$ and the \emph{empirical duality gap} w.r.t. the empirical function $\hat F_S$. 
\citet{zhang2021generalization} established the stability and generalization properties of the empirical saddle point, which is essential to the design of our DP-SMO algorithms.

\section{Differentially Private Stochastic Convex Optimization} \label{sec:sco}

As a warm-up, we start with a simpler problem of differentially private stochastic convex optimization (DP-SCO) to showcase our main ideas. We consider the following stochastic optimization:
\begin{equation*}
    \min_{x\in\cX} \;  F(x) \triangleq \bE_\xi[f(x;\xi)],
\end{equation*}
where the stochastic function $f(\cdot, \xi): \bR^{d}\rightarrow\bR$ is defined on a convex domain $\cX\subset\bR^d$ and $\xi$ is a random vector from an unknown distribution $P_\xi$. Given a dataset $S$ with $n$ i.i.d. samples from $P_\xi$, we develop a generic output perturbation framework with both privacy and utility guarantees on the excess (population) risk.

\subsection{Near-Linear Time Algorithms for Smooth Strongly-Convex Functions}

First we study the case when the objective function is strongly-convex with the following assumptions.
\begin{assumption}
    For any vector $\xi$, the function $f(x;\xi)$ is $L$-Lipschitz, $\ell$-smooth and convex on the closed convex domain $\cX \subset \bR^d$.
    \label{asp:c}
\end{assumption}

\begin{assumption}
    For any vector $\xi$, $f(x;\xi)$ satisfies Assumption \ref{asp:c} and is $\mu$-strongly convex on $\cX$.
    \label{asp:sc}
\end{assumption}

\citet{shalev2009stochastic} proved that the empirical optimal solution has bounded stability if the objective $f(\cdot;\xi)$ is strongly-convex and Lipschitz with respect to its domain, as restated in the following lemma.

\begin{lemma}
    \citep[Theorem 6]{shalev2009stochastic}
    Consider a stochastic optimization problem such that $f(x;\xi)$ is $\mu$-strongly convex and $L$-Lipschitz with respect to $x\in\cX$ for any $\xi$. Given a dataset $S$ with $n$ i.i.d. samples, denote the empirical optimal solution as $\hat x_S^*=\arg\min_{x\in\cX} \hat F_S(x)\triangleq (1/n) \sum_{i=1}^n f(x;\xi_i)$. Then for any neighboring datasets $S\sim S'$,
    \begin{equation*}
        \norm{\hat x_S^* - \hat x_{S'}^*} \leq \frac{2L}{\mu n}.
    \end{equation*}
    The stability result also implies the generalization error of $\hat x_S^*$ can be bounded as 
    \begin{equation*}
        \bE[F(\hat x_S^*) - F(x^*)] \leq \frac{2L^2}{\mu n},
    \end{equation*}
    where $x^*=\arg\min_{x\in\cX} F(x)\triangleq\bE_\xi[f(x;\xi)]$ is the population optimal solution.
    \label{lm:gen-sc}
\end{lemma}

A direct consequence is that we can leverage the stability of ERM to sidestep the sensitivity analysis of optimization algorithms. As long as the output of some algorithm $\cA$ is close enough to the empirical optimal solution, we can show that $\cA$ has bounded sensitivity, and thus standard Gaussian mechanism (Definition \ref{def:gauss}) can be applied to ensure differential privacy. This is formalized in Algorithm \ref{algo:perturb-sc}, where $\cA$ is any algorithm for solving smooth strongly-convex finite-sum minimization problems. The theorem below shows that Algorithm \ref{algo:perturb-sc} is $(\varepsilon, \delta)$-DP with near-optimal guarantees on the excess risk. A proof is provided in Appendix \ref{sec:app-min}.

\begin{algorithm}[t]
    \caption{Output Perturbation for Strongly-Convex Minimization}
    \label{algo:perturb-sc}
    \begin{algorithmic}[1]
        \REQUIRE Dataset $S=\{\xi_i\}_{i=1}^n$, algorithm $\cA$, DP parameters $(\varepsilon, \delta)$, strong-convexity parameter $\mu$.
        \STATE Run the algorithm $\cA$ on the smooth strongly-convex finite-sum problem $\min_{x\in\cX} \hat F_S(x)=(1/n) \sum_{i=1}^n f(x;\xi_i)$ to obtain the output $\cA(S)$ such that $\norm{\cA(S) - \hat x_S^*}\leq L/(\mu n)$ with probability at least $1-\delta/4$, where $\hat x_S^*=\arg\min_{x\in\cX} \hat F_S(x)$ is the empirical optimal solution.
        \ENSURE $\tilde x = \cA(S) + \cN(0, \sigma^2\rI_d)$ with $\sigma = 4L\sqrt{2\log(2.5/\delta)}/(\mu n\varepsilon)$.
    \end{algorithmic}
\end{algorithm}

\begin{theorem}
    Under Assumption \ref{asp:sc}, Algorithm \ref{algo:perturb-sc} is $(\varepsilon, \delta)$-DP and its output $\tilde x$ satisfies
    \begin{align*}
        \text{(excess empirical risk)}\;\;\;\;\;\;\;
        & \bE[\hat F_S(\tilde x) - \hat F_S(\hat x_S^*)] \leq
        33L^2\kappa\cdot \frac{d\log(2.5/\delta)}{\mu n^2\varepsilon^2}, \\
        \text{(excess population risk)}\;\;\;\;\;\;
        & \bE[F(\tilde x) - F(x^*)] \leq
        L^2\kappa\roundBr[\bigg]{\frac{7}{\mu n} +
        \frac{48d\log(2.5/\delta)}{\mu n^2\varepsilon^2}},
    \end{align*}
    where $\kappa=\ell/\mu$ is the condition number and we assume $x^*$ and $\hat x_S^*$ are interior points of $\cX$.
    \label{thm:excess-sc}
\end{theorem}

The utility bounds on the excess risk have an extra $\kappa$ dependence compared to the optimal bound \citep{bassily2014private, bassily2019private}. When the problem is not ill-conditioned, we can achieve the optimal rate. For the time-complexity, Algorithm \ref{algo:perturb-sc} requires high-probability convergence of $\cA$ such that $\norm{\cA(S) - \hat x_S^*}\leq L/(\mu n)$ with probability $1-\delta/4$. Existing algorithms for smooth strongly-convex finite-sum minimization problems output solutions $\cA(S)$ such that
$\bE[\hat F_S(\cA(S))-\hat F_S(\hat x_S^*)] \leq \gamma$
with $\cO(T(n,\kappa)\log(1/\gamma))$ gradient evaluations, where $T(n, \kappa)$ depends on the sample size $n$ and the condition number $\kappa=\ell/\mu$. For example, SVRG \citep{johnson2013accelerating} and SARAH \citep{nguyen2017sarah} have gradient complexity $\cO((n+\kappa)\log(1/\gamma))$ and Katyusha \citep{allen2017katyusha} needs $\cO((n + \sqrt{n\kappa})\log(1/\gamma))$ gradient queries. Since $\bE\norm{\cA(S) - \hat x_S^*} \leq \sqrt{2\gamma/\mu}$ by strong-convexity of $\hat F_S(x)$ and Jensen's inequality, we can then apply Markov's inequality for $\norm{\cA(S) - \hat x_S^*}\geq 0$ and obtain that
\begin{equation*}
    \bP\roundBr*{\norm{\cA(S) - \hat x_S^*} \geq \frac{L}{\mu n}} \leq \frac{n\sqrt{2\gamma\mu}}{L}.
\end{equation*}
Setting $\gamma=\delta^2L^2/(32\mu n^2)$, the RHS becomes $\delta/4$ and the requirement of $\cA$ is satisfied. This implies the complexity of Algorithm \ref{algo:perturb-sc} is $\cO(T(n,\kappa)\log(n/\delta))$.

\begin{remark}
    The gradient complexity of Algorithm \ref{algo:perturb-sc} is $\cO((n+\kappa)\log(n/\delta))$ if $\cA$ is SVRG \citep{johnson2013accelerating} or SARAH \citep{nguyen2017sarah}, and $\cO((n+\sqrt{n\kappa})\log(n/\delta))$ if $\cA$ is Katyusha \citep{allen2017katyusha}.
    \label{rmk:sc}
\end{remark}

As a comparison, most state-of-the-art private algorithms \citep{wang2017differentially, zhang2017efficient, chourasia2021differential} for smooth strongly-convex functions only focus on empirical problems. The linear-time algorithms in \citet{feldman2020private} achieve optimal population guarantees but only work for $\kappa\leq\tilde\cO(n)$ since they utilize the stability of SGD. Instead, we provide a flexible framework that includes various base methods without the necessity to show their algorithm-specific stability.
Further, near-linear time-complexity can be attained using fast variance reduction-based methods. Similar results existed in \citep{lowy2021output}, but their base algorithm $\cA$ only has in-expectation guarantees, which brings a critical challenge to the design of private mechanisms.

\subsection{Near-Linear Time Algorithms for Smooth Convex Functions} \label{sec:c}

Next, we study the convex setting. A direct method is to first reduce the convex problem to a strongly-convex one by adding a regularizer $(\mu/2)\norm{x}^2$ to the objective and then apply Algorithm \ref{algo:perturb-sc}. However, this approach only achieves a sub-optimal rate \citep{zhang2017efficient, lowy2021output}. Inspired by the phased-ERM \citep{feldman2020private} method for nonsmooth convex losses, we show that a more sophisticated multi-phase Algorithm \ref{algo:perturb-phased_c} using an increasing sequence of regularization parameters $\{\mu_k\}$ can achieve near-optimal guarantees on the population loss for smooth DP-SCO.
The algorithm applies the following stability and generalization results of ERM \citep{shalev2009stochastic} that also hold for regularized empirical problems. It is worth mentioning that Lemma \ref{lm:gen-c} does not require the Lipschitzness of the regularizer, and is not a trivial extension of its unregularized version.

\begin{lemma}
    \citep[Theorem 7]{shalev2009stochastic}
    Under the same settings as Lemma \ref{lm:gen-sc}. Consider the case when $f(x;\xi)$ is convex and $L$-Lipschitz with a $\mu$-strongly convex regularizer $G(x)$. Denote the empirical optimal solution as $\hat x_S^*=\arg\min_{x\in\cX}\{\hat F_S(x) + G(x)\}$. Then for any neighboring datasets $S\sim S'$, we have that
    \begin{equation*}
        \norm{\hat x_S^* - \hat x_{S'}^*} \leq \frac{2L}{\mu n}.
    \end{equation*}
    The stability result also implies the generalization error of the empirical solution:
    \begin{equation*}
        \bE\squareBr[\bigg]{F(\hat x_S^*) + G(\hat x_S^*)} -
        \bE\squareBr*{\min_{x\in\cX}\{F(x) + G(x)\}} \leq \frac{2L^2}{\mu n},
    \end{equation*}
    measured by the excess population risk.
    \label{lm:gen-c}
\end{lemma}

\begin{algorithm}[t]
    \caption{Phased Output Perturbation for Convex Minimization}
    \label{algo:perturb-phased_c}
    \begin{algorithmic}[1]
        \REQUIRE Dataset $S=\{\xi_i\}_{i=1}^n$, algorithm $\cA$, DP parameters $(\varepsilon, \delta)$, regularizer $\mu$, initializer $x_0$.
        \STATE Set $K=\log(n)$, $\bar n=n/K$ and $\tilde x_0=x_0$.
        \FOR{$k=1,\cdots,K$}
            \STATE Set $\mu_k=\mu\cdot 2^k$.
            \STATE Run the algorithm $\cA$ on the smooth strongly-convex finite-sum minimization problem
            $
            \min_{x\in\cX} \hat F_k(x) \triangleq (1/\bar n)
            \sum_{i=(k-1)\bar n + 1}^{k\bar n} f(x;\xi_i) +
            (\mu_k/2) \norm{x - \tilde x_{k-1}}^2
            $
            to obtain the output $x_k$ such that
            $\norm{x_k - \hat x_k^*}\leq L/(\mu_k \bar n)$ with probability at least $1-\delta/4$, where $\hat x_k^*=\arg\min_{x\in\cX} \hat F_k(x)$ is the empirical optimal solution.
            \STATE $\tilde x_k=x_k+\cN(0, \sigma_k^2\rI_d)$ with $\sigma_k=4L\sqrt{2\log(2.5/\delta)}/(\mu_k\bar n\varepsilon)$.
        \ENDFOR
        \ENSURE $\tilde x_K$.
    \end{algorithmic}
\end{algorithm}

In Algorithm \ref{algo:perturb-phased_c}, the increasing regularization parameters ensure that both added DP noise and approximation errors coming from the regularizer can be properly controlled. Here we \emph{interactively} \citep{dwork2014algorithmic} access the dataset $S$ multiple times where a future output is allowed to depend on all the past outputs. However, since each phase only accesses a disjoint partition of the dataset, we can use the parallel composition in Lemma \ref{lm:composition} to guarantee differential privacy. The utility analysis of Algorithm \ref{algo:perturb-phased_c} is similar to phased-ERM \citep{feldman2020private}. The core is the following lemma that leverages the generalization properties of regularized empirical problems.

\begin{lemma}
    Let Assumption \ref{asp:c} hold. For $1\leq k\leq K$, by the settings in Algorithm \ref{algo:perturb-phased_c}, we have that
    \begin{equation*}
        \bE[F(\hat x_k^*) - F(\hat x_{k-1}^*)] \leq
        \frac{\mu_k}{2} \bE\norm{\hat x_{k-1}^* - \tilde x_{k-1}}^2 + \frac{2L^2}{\mu_k\bar n},
    \end{equation*}
    where $\hat x_k^*$ is the optimal solution of the regularized empirical function $\hat F_k(x)$ and $\hat x_0^*$ is defined later in the proof of Theorem \ref{thm:excess-c} for the guarantees of Algorithm \ref{algo:perturb-phased_c}.
    \label{lm:gen-algo_c}
\end{lemma}

\begin{proof}
    Applying the generalization results in Lemma \ref{lm:gen-c} for $\hat F_k(x)$ with regularization term $(\mu_k/2)\norm{x-\tilde x_{k-1}}^2$ and dataset $S_k:=\{\xi_i\}_{i=(k-1)\bar n+1}^{k\bar n}$, we have that for any $x\in\cX$,
    \begin{align*}
        \bE\squareBr*{F(\hat x_k^*) +
        \frac{\mu_k}{2}\norm{\hat x_k^*-\tilde x_{k-1}}^2}
        & \leq
        \bE\squareBr*{\min_{x'\in\cX}\curlyBr*{F(x') +
        \frac{\mu_k}{2}\norm{x'-\tilde x_{k-1}}^2}} +
        \frac{2L^2}{\mu_k \bar n} \\
        & \leq
        \bE\squareBr*{F(x) + \frac{\mu_k}{2}\norm{x-\tilde x_{k-1}}^2} +
        \frac{2L^2}{\mu_k \bar n}.
    \end{align*}
    Setting $x=\hat x_{k-1}^*$, the proof is done since
    $\norm{\hat x_k^*-\tilde x_{k-1}}^2\geq 0$.
\end{proof}

With above lemma, we can show that the output of Algorithm \ref{algo:perturb-phased_c} satisfies the following guarantees.

\begin{theorem}
    Let Assumption \ref{asp:c} hold. Suppose there exists at least one optimal solution $x^*\in\arg\min_{x\in\cX} F(x)$ such that $\norm{x^*}\leq D$. Algorithm \ref{algo:perturb-phased_c} is $(\varepsilon, \delta)$-DP and its output $\tilde x_K$ satisfies
    \begin{equation*}
        \bE[F(\tilde x_K) - F(x^*)] \leq
        4LD\cdot\log(n)\bigg(\frac{1}{\sqrt{n}} + \frac{7\sqrt{d\log(2.5/\delta)}}{n\varepsilon}\bigg),
    \end{equation*}
    for the excess population risk when setting 
    $\mu=({L}/{D})\max\big\lbrace{1}/{\sqrt{n}}, {14\log(n)\sqrt{d\log(2.5/\delta)}}/{(n\varepsilon)}\big\rbrace$.
    \label{thm:excess-c}
\end{theorem}

\begin{proof}[Proof Sketch]
    First by Theorem \ref{thm:excess-sc}, we know each phase is $(\varepsilon, \delta)$-DP and the output $\tilde x_k$ satisfies that
    \begin{equation*}
        \bE\norm{\tilde x_k-\hat x_k^*}^2 \leq 65L^2\cdot\frac{d\log(2.5/\delta)}{\mu_k^2\bar n^2\varepsilon^2}.
    \end{equation*}
    Then by the parallel composition in Lemma \ref{lm:composition}, Algorithm \ref{algo:perturb-phased_c} is $(\varepsilon, \delta)$-DP. For the utility bound of $\tilde x_K$, we start with the following error decomposition:
    \begin{equation*}
        \bE[F(\tilde x_K) - F(x^*)]
        = \bE[F(\tilde x_K) - F(\hat x_K^*)] +
        \sum_{k=1}^K \bE[F(\hat x_k^*) - F(\hat x_{k-1}^*)],
    \end{equation*}
    where $x^*\in\arg\min_{x\in\cX} F(x)$ is the population optimal solution, $\hat x_k^*$ is the optimal solution of the regularized empirical function $\hat F_k(x)$ in Algorithm \ref{algo:perturb-phased_c}, and we let
    $\hat x_0^*=x^*$ only for simplicity of the analysis. The first term in the RHS of above equation is bounded by Lipschitzness of $F(x)$, and the second term uses Lemma \ref{lm:gen-algo_c}. Therefore,
    \begin{align*}
        \bE[F(\tilde x_K) - F(x^*)]
        & \leq
        L\sqrt{\bE\norm{\tilde x_K-\hat x_K^*}^2} +
        \sum_{k=1}^K \roundBr*{
        \frac{\mu_k}{2} \bE\norm{\hat x_{k-1}^* - \tilde x_{k-1}}^2 + \frac{2L^2}{\mu_k\bar n}} \\
        & \leq
        9LD\cdot\frac{\sqrt{d\log(2.5/\delta)}}{\bar n\varepsilon} +
        \sum_{k=2}^K \frac{65\mu_k}{2\mu_{k-1}}
        \frac{L^2 d\log(2.5/\delta)}{\mu_{k-1} \bar n^2\varepsilon^2} +
        \sum_{k=1}^K\frac{2L^2}{\mu_k\bar n} + \frac{\mu_1}{2}\norm{x^*-x_0}^2 \\
        & \leq
        4LD\cdot\log(n)\roundBr*{\frac{1}{\sqrt{n}} + \frac{7\sqrt{d\log(2.5/\delta)}}{n\varepsilon}},
    \end{align*}
    by the settings of $\mu_k$ and $\mu$. The detailed proof is given in Appendix \ref{sec:app-min}.
\end{proof}

The utility bound on the excess population risk has an extra logarithmic term in $n$, which can be removed by a different parameter choice. For example in phased-ERM \citep{feldman2020private}, the regularization parameter increases as $\mu_k=\mu\cdot 2^{3k}$ across different phases, and the size of partitioned datasets decreases as $n_k=n/2^k$. However, the decreasing data size may not give us optimal algorithms for convex-concave minimax problems. To be consistent, we also use a fixed data size for the convex minimization case, despite this additional logarithmic factor.

The gradient complexity of Algorithm \ref{algo:perturb-phased_c} is the total complexity of $\cA$ to solve the smooth strongly-convex finite-sum problem at each phase. Remark \ref{rmk:sc} suggests that the complexity of each phase is $\cO(T(\bar n, (\ell+\mu_k)/\mu_k)\log(1/\gamma_k))$ with $\gamma_k=\delta^2 L^2/(32\mu_k\bar n^2)$ when solving a $\mu_k$-strongly convex, $(\ell+\mu_k)$-smooth finite-sum problems with sample size $\bar n$, so the total complexity is $\sum_{k=1}^K\cO(T(\bar n, \ell/\mu_k + 1)\log(n/\delta))$. Therefore, the complexity is
\begin{equation*}
    \sum_{k=1}^K  \cO\roundBr*{
    \roundBr*{\bar n + \frac{\ell}{\mu_k} + 1}\log(n/\delta)} = \cO((n+\sqrt{n}\ell D/L)\log(n/\delta)),
\end{equation*}
for SVRG \citep{johnson2013accelerating} and SARAH \citep{nguyen2017sarah} since $\sum_{k=1}^K 1/\mu_k \leq 1/\mu \leq \cO(\sqrt{n}D/L)$ by the settings of $\mu_k$ and $\mu$ in Algorithm \ref{algo:perturb-phased_c}. Similarly for Katyusha \citep{allen2017katyusha}, we can compute that the total complexity is $\cO((n+n^{3/4}\sqrt{\ell D/L})\log(n/\delta))$ since $\sum_{k=1}^K 1/\sqrt{\mu_k}\leq\cO(n^{1/4}\sqrt{D/L})$. We also point out that the smoothness assumption here is not necessary to derive the utility bound, but allows the use of fast accelerated algorithms for solving the regularized ERM problems.

\begin{remark}
    The gradient complexity of Algorithm \ref{algo:perturb-phased_c} is $\cO((n+\sqrt{n}\ell D/L)\log(n/\delta))$ if $\cA$ is SVRG \citep{johnson2013accelerating} or SARAH \citep{nguyen2017sarah}, and $\cO((n+n^{3/4}\sqrt{\ell D/L})\log(n/\delta))$ if $\cA$ is Katyusha \citep{allen2017katyusha}.
\end{remark}

\citet{bassily2019private} first proved optimal population guarantees for smooth convex functions, but their algorithm needs $\cO(n^{3/2}\sqrt{\varepsilon})$ gradient queries. The algorithms in \citet{feldman2020private} achieve linear time-complexity, and require that $\ell\leq\cO((L/D)\max\{\sqrt{n}, \sqrt{d\log(1/\delta)}/\varepsilon\})$ rooted in the stability analysis of SGD. In contrast, our framework can achieve near-optimal population guarantees using any algorithm without additional restrictions to the smoothness parameter. By equipping with variance reduction-based methods, near-linear time-complexity can also be attained when $\ell\leq\cO(\sqrt{n}L/D)$. To the best of our knowledge, this is the first time that sophisticated optimization algorithms besides SGD are proven to obtain near-optimal population guarantees in near-linear time. Additionally, in the regime that $\Omega(\sqrt{n}L/D)\leq\ell \leq\cO(nL/D)$ where previous smooth DP-SCO algorithms \citep{bassily2019private, feldman2020private} fail to provide optimal guarantees, our framework still achieves a near-optimal rate with a better gradient complexity $\tilde\cO(n^{3/4}\sqrt{\ell D/L})$ compared to the state-of-the-art nonsmooth DP-SCO algorithms \citep{asi2021private, kulkarni2021private}.

\section{Differentially Private Stochastic Minimax Optimization} \label{sec:smo}

Using the same ideas as the minimization case in the previous section, we develop differentially private algorithms for stochastic minimax optimization:
\begin{equation*}
    \min_{x\in\cX} \max_{y\in\cY} \; F(x,y) = \bE_{\xi}[f(x,y;\xi)],
\end{equation*}
with utility guarantees on the (population) duality gap.

\subsection{Near-Linear Time Algorithms for Smooth Strongly-Convex--Strongly-Concave Functions}

First, we study the strongly-convex--strongly-concave (SC-SC) case with the following assumptions.

\begin{assumption}
    For any $\xi$, $f(x,y;\xi)$ is $L$-Lipschitz and $\ell$-smooth on the closed convex domain $\cX \times \cY \subset \bR^{d_x} \times \bR^{d_y}$. Moreover, $f(\cdot,y;\xi)$ is convex on $\cX$ for any $y\in\cY$, and $f(x,\cdot;\xi)$ is concave on $\cY$ for any $x\in\cX$.
    \label{asp:cc}
\end{assumption}

\begin{assumption}
    For any $\xi$, $f(x,y;\xi)$ satisfies Assumption \ref{asp:cc} and $f(\cdot,y;\xi)$ is $\mu_x$-strongly convex on $\cX$ for any $y\in\cY$, and $f(x,\cdot;\xi)$ is $\mu_y$-strongly concave on $\cY$ for any $x\in\cX$.
    \label{asp:scsc}
\end{assumption}

In a manner similar to the minimization case, \citet{zhang2021generalization} showed that the optimal solution to the empirical saddle point problem is stable, which also implies its generalization error. Here we use a slightly tighter stability bound than theirs and provide a proof in Appendix \ref{sec:app-minimax} for completeness.

\begin{lemma}
    \citep[Lemma 1 and Theorem 1]{zhang2021generalization}
    Consider a stochastic minimax problem such that $f(x,y;\xi)$ is $\mu_x$-strongly convex $\mu_y$-strongly concave and $L$-Lipschitz with respect to $x\in\cX$ and $y\in\cY$. Let $\mu=\min\{\mu_x,\mu_y\}$ and denote the empirical saddle point of function $\hat F_S(x,y)$ as $(\hat x_S^*, \hat y_S^*)$ given dataset $S$ with $n$ i.i.d. samples. Then for any neighboring datasets $S\sim S'$, we have that
    \begin{equation*}
        \mu_x\norm{\hat x_S^* - \hat x_{S'}^*}^2 +
        \mu_y\norm{\hat y_S^* - \hat y_{S'}^*}^2 \leq
        \frac{4L^2}{\mu n^2}.
    \end{equation*}
    The stability result implies the generalization error of the empirical solution can be bounded as
    \begin{equation*}
        \max_{y\in\cY}\bE[F(\hat x_S^*, y)] -
        \min_{x\in\cX}\bE[F(x, \hat y_S^*)] \leq \frac{2\sqrt{2}L^2}{\mu n},
    \end{equation*}
    measured by the population weak duality gap.
    \label{lm:gen-scsc}
\end{lemma}

\begin{algorithm}[t]
    \caption{Output Perturbation for Strongly-Convex--Strongly-Concave Minimax Problems}
    \label{algo:perturb-scsc}
    \begin{algorithmic}[1]
        \REQUIRE Dataset $S=\{\xi_i\}_{i=1}^n$, algorithm $\cA$, DP parameters $(\varepsilon, \delta)$, SC parameters $(\mu_x, \mu_y)$.
        \STATE Run the algorithm $\cA$ on the smooth strongly-convex--strongly-concave finite-sum saddle point problem
        $
        \min_{x\in\cX} \max_{y\in\cY} \hat F_S(x,y) =
        (1/n) \sum_{i=1}^n f(x, y;\xi_i)
        $
        to obtain the output $(\cA_x(S), \cA_y(S))$ such that with probability at least $1-\delta/8$,
        \begin{equation}
            \mu_x\norm{\cA_x(S)-\hat x_S^*}^2 +
            \mu_y\norm{\cA_y(S)-\hat y_S^*}^2 \leq
            \frac{L^2}{\mu n^2},
            \label{eq:scsc-condition}
        \end{equation}
        where $(\hat x_S^*, \hat y_S^*)$ is the saddle point of $\hat F_S(x,y)$ and we let
        $\mu:=\min\{\mu_x, \mu_y\}$.
        \STATE Set
        $
        \sigma_x=(8L/(n\varepsilon))\sqrt{2\log(5/\delta)/(\mu_x\mu)}
        $ and
        $
        \sigma_y=(8L/(n\varepsilon))\sqrt{2\log(5/\delta)/(\mu_y\mu)}
        $.
        \ENSURE $\tilde x = \cA_x(S) + \cN(0,\sigma_x^2\rI_{d_x})$ and
        $\tilde y = \cA_y(S) + \cN(0,\sigma_y^2\rI_{d_y})$.
    \end{algorithmic}
\end{algorithm}

As a direct consequence, any algorithm $\cA$ whose output is sufficiently close to the empirical saddle point has bounded sensitivity. This observation leads to Algorithm \ref{algo:perturb-scsc} for SC-SC DP-SMO with guarantees given in the theorem below. Here $\cA$ can be any method for smooth SC-SC finite-sum minimax problems, and smoothness allows us to obtain efficient algorithms. The proof of Theorem \ref{thm:scsc} can be found in Appendix \ref{sec:app-minimax}.

\begin{theorem}
    \label{thm:scsc}
    Under Assumption \ref{asp:scsc}. Let saddle points
    $(\hat x_S^*, \hat y_S^*)$ and $(x^*, y^*)$ be interior points of the domain $\cX\times\cY$. Then Algorithm \ref{algo:perturb-scsc} is $(\varepsilon, \delta)$-DP and its output $(\tilde x, \tilde y)$ satisfies the following utility bounds on the empirical and population strong duality gap:
    \begin{align*}
        \parbox{7em}{(empirical)}
        & \bE\squareBr[\bigg]{\max_{y\in\cY}\hat F_S(\tilde x, y) - 
        \min_{x\in\cX}\hat F_S(x, \tilde y)}
        \leq 257L^2(\kappa_x\kappa_y + \kappa)
        \frac{d\log(5/\delta)}{\mu n^2\varepsilon^2}, \\
        \parbox{7.3em}{(population)}
        & \bE\squareBr[\bigg]{\max_{y\in\cY} F(\tilde x, y) -
        \min_{x\in\cX} F(x, \tilde y)}
        \leq 3L^2 (\kappa_x\kappa_y + \kappa)
        \roundBr[\bigg]{\frac{3}{\mu n} +
        \frac{128d\log(5/\delta)}{\mu n^2\varepsilon^2}},
    \end{align*}
    where we let $\mu=\min\{\mu_x, \mu_y\}$, $\kappa_x=\ell/\mu_x$, $\kappa_y=\ell/\mu_y$, $\kappa = \ell/\mu$ , and $d=\max\{d_x, d_y\}$.
\end{theorem}

The utility bound on the strong duality gap scales with $\cO\roundBr*{\kappa^2\roundBr*{1/(\mu n) +
d\log(1/\delta)/(\mu n^2\varepsilon^2)}}$.
Since the minimax problem is equivalent to a minimization problem on $x$ when the domain $\cY$ is restricted to a singleton, the lower-bound $\Omega(1/(\mu_x n)+d_x\log(1/\delta)/(\mu_x n^2\varepsilon^2))$ of SC DP-SCO \citep{bassily2014private, bassily2019private} trivially holds for SC-SC DP-SMO, and our results are near-optimal w.r.t. $n$ and $(\varepsilon, \delta)$. It remains unclear whether the dependence on $\mu_x$, $\mu_y$ and $d_x$, $d_y$ can be further improved as the exact lower-bound for SC-SC DP-SMO is still an open question. Nonetheless, we achieve near-optimal rate on the strong duality gap when the problem is not ill-conditioned.

For the gradient complexity, Algorithm \ref{algo:perturb-scsc} requires that \eqref{eq:scsc-condition} holds with probability at least $1-\delta/8$. We first review some existing methods for smooth SC-SC finite-sum saddle point problems. For an $\ell$-smooth, $\mu_x$-strongly convex and $\mu_y$-strongly concave finite-sum minimax problem $\min_{x\in\cX}\max_{y\in\cY} \hat F_S(x,y)=(1/n)\sum_{i=1}^n f(x,y;\xi_i)$, these algorithms output a $\gamma$-approximate saddle point $(\cA_x(S), \cA_y(S))$ such that
\begin{equation*}
    \bE\squareBr*{\max_{y\in\cY}\hat F_S(\cA_x(S),y)-\min_{x\in\cX}\hat F_S(x,\cA_y(S))} \leq \gamma,    
\end{equation*}
with $\cO(T(n, \kappa_x, \kappa_y)\log(1/\gamma))$ gradient complexity, where the expectation is taken with respect to the randomness in the algorithm and $\kappa_x=\ell/\mu_x$, $\kappa_y=\ell/\mu_y$.

GDA and Extragradient \citep{tseng1995linear} use the full-batch gradient at each iteration and consider a deterministic problem. Since the iteration complexity is $\cO(\kappa^2\log(1/\gamma))$ and $\cO(\kappa\log(1/\gamma))$ respectively, the total gradient complexity is $\cO(n\kappa^2\log(1/\gamma))$ for GDA and $\cO(n\kappa\log(1/\gamma))$ for Extragradient to achieve a $\gamma$-approximate saddle point, where $\kappa=\max\{\kappa_x, \kappa_y\}$. These two algorithms do not leverage the finite-sum structure and obtain sub-optimal convergence rates. \citet{palaniappan2016stochastic} first introduced the use of variance reduction methods into finite-sum minimax optimization problems. The gradient complexity is improved to $\cO((n+\kappa^2)\log(1/\gamma))$ by SVRG/SAGA and $\cO((n+\sqrt{n}\kappa)\log(1/\gamma))$ by the accelerated SVRG/SAGA. The state-of-the-art complexity $\cO((n+\sqrt{n\kappa_x\kappa_y}+n^{3/4}\sqrt{\kappa})\log(1/\gamma))$ that also matches with the lower-bound is provided in \citet{yang2020catalyst} and \citet{luo2021near}. Their algorithms are based on a catalyst framework.

Since $\bE[\mu_x\norm{\cA_x(S)-\hat x_S^*}^2 + \mu_y\norm{\cA_y(S)-\hat y_S^*}^2] \leq 2\gamma$ by strong-convexity and optimality condition, we can then apply Markov's inequality and obtain that
\begin{equation*}
    \bP\roundBr*{\roundBr*{\mu_x\norm{\cA_x(S)-\hat x_S^*}^2 + \mu_y\norm{\cA_y(S)-\hat y_S^*}^2} \geq \frac{L^2}{\mu n^2}} \leq \frac{2\gamma\mu n^2}{L^2}.
\end{equation*}
Setting $\gamma=\delta L^2/(16\mu n^2)$, the RHS becomes $\delta/8$ and the requirement of $\cA$ is satisfied. This implies the complexity of Algorithm \ref{algo:perturb-scsc} is $\cO(T(n, \kappa_x, \kappa_y)\log(n/\delta))$.

\begin{remark}
    The gradient complexity of Algorithm \ref{algo:perturb-scsc} is $\cO(n\kappa\log(n/\delta))$ for Extragradient \citep{tseng1995linear}, $\cO((n+\kappa^2)\log(n/\delta))$ for SVRG/SAGA \citep{palaniappan2016stochastic}, $\cO((n+\sqrt{n}\kappa)\log(n/\delta))$ for Acc-SVRG/SAGA \citep{palaniappan2016stochastic} and $\cO((n+\sqrt{n\kappa_x\kappa_y}+n^{3/4}\sqrt{\kappa})\log(n/\delta))$ for AL-SVRE \citep{luo2021near} and Catalyst-Acc-SVRG \citep{yang2020catalyst}. See Table \ref{tab:complexity} for a summary.
    \label{rmk:scsc}
\end{remark}

The flexible framework allows the use of off-the-shelf optimization algorithms for smooth SC-SC finite-sum minimax problems from a well-studied research community \citep{yang2020catalyst, luo2021near} without being aware of the algorithm-specific stability bound. Based on this framework, we can produce the first near-linear time algorithms for smooth SC-SC DP-SMO with near-optimal guarantees on the strong duality gap. Similar rates were obtained in \citep{lowy2021output} on the population primal risk through the reduction to SC DP-SCO. However, the generalization error on the primal function does not always apply to the original minimax problem when the expectation and maximization cannot be exchanged.

\subsection{Near-Linear Time Algorithms for Smooth Convex-Concave Functions} \label{sec:cc}

In this section, we focus on the more general case when the objective is smooth and convex-concave. We continue to build private algorithms using the stability and generalization results of empirical solutions. When the function is convex-concave, the following results for the regularized empirical problems can be applied by adding an SC-SC regularizer. Note that Lemma \ref{lm:gen-cc} does not require Lipschitzness of the regularizer and is not a trivial extension of Lemma \ref{lm:gen-scsc}.

\begin{lemma}
    \citep[Lemma 2 and 3]{zhang2021generalization}
    Under the same settings as Lemma \ref{lm:gen-scsc}. Consider the case when $f(x,y;\xi)$ is convex-concave and $L$-Lipschitz with a $\mu_x$-strongly convex $\mu_y$-strongly concave regularizer $G(x,y)$. Let $\mu=\min\{\mu_x,\mu_y\}$ and denote the empirical saddle point of function $\hat F_S(x,y) + G(x,y)$ as $(\hat x_S^*, \hat y_S^*)$. Then for any neighboring datasets $S\sim S'$, we have that
    \begin{equation*}
        \mu_x\norm{\hat x_S^* - \hat x_{S'}^*}^2 +
        \mu_y\norm{\hat y_S^* - \hat y_{S'}^*}^2 \leq
        \frac{4L^2}{\mu n^2}.
    \end{equation*}
    The stability result implies the generalization error of the empirical solution can be bounded as
    \begin{equation*}
        \max_{y\in\cY}\bE\squareBr[\Big]{F(\hat x_S^*, y) + G(\hat x_S^*, y)} -
        \min_{x\in\cX}\bE\squareBr[\Big]{F(x, \hat y_S^*) + G(x, \hat y_S^*)} \\
        \leq \frac{2\sqrt{2}L^2}{\mu n},
    \end{equation*}
    measured by the population weak duality gap.
    \label{lm:gen-cc}
\end{lemma}

The following direct corollary of Lemma \ref{lm:gen-cc} will be useful in the analysis of our private algorithms.

\begin{corollary}
    Under the same settings as Lemma \ref{lm:gen-cc}. Let $(u_S,v_S) \in\cX\times\cY$ be some points that may have dependence on the dataset $S$. Suppose their stability with respect to $S$ is bounded as $\bE\norm{u_S-u_{S'}}\leq\Delta_u$ and $\bE\norm{v_S-v_{S'}}\leq\Delta_v$ for any neighboring datasets $S\sim S'$. Then it holds that
    \begin{equation*}
        \bE[F(\hat x_S^*, v_S) - F(u_S, \hat y_S^*)] \leq
        \bE[G(u_S, \hat y_S^*) - G(\hat x_S^*, v_S)] +
        \frac{2\sqrt{2}L^2}{\mu n} + L(\Delta_u + \Delta_v).
    \end{equation*}
    As a special case, when $u$ and $v$ are independent of $S$, the third term vanishes since $\Delta_u=\Delta_v=0$.
    \label{cor:gen-cc}
\end{corollary}

\begin{algorithm}[t]
    \caption{Phased Output Perturbation for Convex-Concave Minimax Problems}
    \label{algo:perturb-phased_cc}
    \begin{algorithmic}[1]
        \REQUIRE Dataset $S=\{\xi_i\}_{i=1}^n$, algorithm $\cA$, DP Parameters $(\varepsilon, \delta)$, regularizer $\mu$, initializer $x_0$.
        \STATE Set $K=\log(n)$, $\bar n=n/K$ and $\tilde x_0 = x_0$.
        \FOR{$k=1,\cdots,K$}
            \STATE Set $\mu_k=\mu\cdot 2^k$.
            \STATE Run the algorithm $\cA$ on the smooth SC-SC finite-sum saddle point problem
            \begin{equation*}
                \min_{x\in\cX} \max_{y\in\cY} \; \hat F_k(x,y) \triangleq
                \frac{1}{\bar n} \sum_{i=(k-1)\bar n+1}^{k\bar n} f(x,y;\xi_i) +
                \frac{\mu_k}{2}\norm{x - \tilde x_{k-1}}^2 - \frac{\mu}{2}\norm{y}^2,
            \end{equation*}
            to obtain the output $(x_k, y_k)$ such that with probability $1-\delta/8$,
            \begin{equation*}
                \mu_k\norm{x_k - \hat x_k^*}^2 + \mu\norm{y_k - \hat y_k^*}^2 \leq
                \frac{L^2}{\mu \bar n^2},
            \end{equation*}
            where $(\hat x_k^*, \hat y_k^*)$ is the saddle point of the regularized empirical function $\hat F_k(x,y)$.
            \STATE $\tilde x_k=x_k + \cN(0, \sigma_k^2 \rI_{d_x})$ with $\sigma_k=(8L/(\bar n \varepsilon)) \sqrt{2\log(5/\delta)/(\mu_k\mu)}$.
        \ENDFOR
        \ENSURE $\tilde x_K$.
    \end{algorithmic}
\end{algorithm}

Unlike DP-SCO, the interaction between the primal variable $x$ and the dual variable $y$ puts additional challenges. In order to derive a near-optimal algorithm for convex-concave DP-SMO, we need to carefully design the regularization parameters for the primal $x$ and the dual $y$ to control both approximation errors and DP noise. Based on the phased output perturbation framework, we propose a novel algorithm for smooth convex-concave DP-SMO in Algorithm \ref{algo:perturb-phased_cc}. Here, the small fixed parameter $\mu$ ensures the approximation error from adding regularization for the dual is properly bounded. The role of the increasing parameter $\mu_k$ is the same as Algorithm \ref{algo:perturb-phased_c}. The analysis of Algorithm \ref{algo:perturb-phased_cc} is also similar to Algorithm \ref{algo:perturb-phased_c}. By Theorem \ref{thm:scsc}, we can first obtain the privacy and utility guarantees of each phase. Then parallel composition in Lemma \ref{lm:composition} guarantees that Algorithm \ref{algo:perturb-phased_cc} is $(\varepsilon/2, \delta/2)$-DP since the datasets in different phases are disjoint. For the utility guarantee of the output $\tilde x_K$, we use the lemma below that is a consequence of Corollary \ref{cor:gen-cc}.

\begin{lemma}
    Let Assumption \ref{asp:cc} hold. For $1\leq k\leq K$, by the settings in Algorithm \ref{algo:perturb-phased_cc}, we have that
    \begin{equation*}
        \bE[F(\hat x_k^*, \hat y_{k+1}^*) - F(\hat x_{k-1}^*,\hat y_k^*)]
        \leq \frac{\mu_k}{2}\bE\norm{\hat x_{k-1}^*-\tilde x_{k-1}}^2 + \frac{8L^2}{\mu \bar n} + \frac{\mu}{2}\bE\squareBr*{
        \norm{\hat y_{k+1}^*}^2 - \norm{\hat y_k^*}^2},
    \end{equation*}
    where $(\hat x_k^*, \hat y_k^*)$ is the saddle point of the regularized empirical function $\hat F_k(x,y)$, and $\hat x_0^*$, $\hat y_{K+1}^*$ are defined later in the proof of Theorem \ref{thm:cc}.
    \label{lm:gen-algo_cc}
\end{lemma}

\begin{proof}[Proof Sketch]
    Applying Corollary \ref{cor:gen-cc} for $\hat F_k(x,y)$ with regularization term $(\mu_k/2)\norm{x-\tilde x_{k-1}}^2 - (\mu/2)\norm{y}^2$ and dataset $S_k:=\{\xi_i\}_{i=(k-1)\bar n+1}^{k\bar n}$, we have that for any $x\in\cX$ and $y\in\cY$,
    \begin{equation*}
        \bE[F(\hat x_k^*, y) - F(x, \hat y_k^*)]
        \leq
        \frac{\mu_k}{2}\bE\norm{x-\tilde x_{k-1}}^2 +
        \frac{\mu}{2}\bE\norm{y}^2 - \frac{\mu}{2}\bE\norm{\hat y_k^*}^2 +
        \frac{2\sqrt{2}L^2}{\mu\bar n} + L(\Delta_x + \Delta_y),
    \end{equation*}
    where $\Delta_x$ and $\Delta_y$ are the stability bounds of $x$ and $y$ with respect to $S_k$. When setting $x=\hat x_{k-1}^*$ and $y=\hat y_{k+1}^*$, the proof is complete since $\Delta_x=0$ and $\Delta_y\leq\sqrt{26}L/(\mu\bar n)$. More details are provided in Appendix \ref{sec:app-minimax}.
\end{proof}

\begin{algorithm}[t]
    \caption{Phased Output Perturbation for Convex-Concave Minimax Problems}
    \label{algo:perturb-phased_cc-II}
    \begin{algorithmic}[1]
        \REQUIRE Dataset $S=\{\xi_i\}_{i=1}^n$, algorithm $\cA$, DP Parameters $(\varepsilon, \delta)$, regularizer $\mu$, initializer $y_0$.
        \STATE Set $K=\log(n)$, $\bar n=n/K$ and $\tilde y_0 = y_0$.
        \FOR{$k=1,\cdots,K$}
            \STATE Set $\mu_k=\mu\cdot 2^k$.
            \STATE Run the algorithm $\cA$ on the smooth SC-SC finite-sum saddle point problem
            \begin{equation*}
                \min_{x\in\cX} \max_{y\in\cY} \; \hat F_k(x,y) \triangleq
                \frac{1}{\bar n} \sum_{i=(k-1)\bar n+1}^{k\bar n}
                f(x,y;\xi_i) + \frac{\mu}{2}\norm{x}^2 -
                \frac{\mu_k}{2}\norm{y - \tilde y_{k-1}}^2,
            \end{equation*}
            to obtain the output $(x_k, y_k)$ such that with probability $1-\delta/8$,
            \begin{equation*}
                \mu\norm{x_k - \hat x_k^*}^2 + \mu_k\norm{y_k - \hat y_k^*}^2 \leq \frac{L^2}{\mu \bar n^2},
            \end{equation*}
            where $(\hat x_k^*, \hat y_k^*)$ is the saddle point of the regularized empirical function $\hat F_k(x,y)$.
            \STATE $\tilde y_k=y_k + \cN(0, \sigma_k^2 I_{d_y})$ with $\sigma_k=(8L/(\bar n\varepsilon)) \sqrt{2\log(5/\delta)/(\mu_k\mu)}$.
        \ENDFOR
        \ENSURE $\tilde y_K$.
    \end{algorithmic}
\end{algorithm}

We only add noise to the primal solutions and output $\tilde x_K$ in Algorithm \ref{algo:perturb-phased_cc}. A natural way to derive a corresponding dual solution is to solve the smooth concave maximization problem $\max_{y\in\cY} \bE[F(\tilde x_K, y)]$ with DP constraints. We instead provide an alternative method that is symmetric to Algorithm \ref{algo:perturb-phased_cc} by switching the role of the primal and the dual. This phased method in Algorithm \ref{algo:perturb-phased_cc-II} only perturbs the dual variables and outputs $\tilde y_K$. Without causing confusion, we borrow notations from Algorithm \ref{algo:perturb-phased_cc} for simplicity. Similarly we can prove that Algorithm \ref{algo:perturb-phased_cc-II} is $(\varepsilon/2, \delta/2)$-DP by parallel composition. Therefore by basic composition in Lemma \ref{lm:composition}, the composition $(\tilde x_K, \tilde y_K)$ of Algorithm \ref{algo:perturb-phased_cc} and \ref{algo:perturb-phased_cc-II} is $(\varepsilon, \delta)$-DP. The utility guarantees are shown in the following theorem.

\begin{theorem}
    \label{thm:cc}
    Let Assumption \ref{asp:cc} hold. Suppose $\max\{\norm{x},\norm{y}\}\leq D$ for all $x\in\cX$ and $y\in\cY$. Then the composition of Algorithm \ref{algo:perturb-phased_cc} and \ref{algo:perturb-phased_cc-II} is $(\varepsilon, \delta)$-DP and the output $(\tilde x_K, \tilde y_K)$ satisfies the following bound on the population weak duality gap:
    \begin{equation*}
        \max_{y\in\cY}\bE[F(\tilde x_K, y)] -
        \min_{x\in\cX}\bE[F(x, \tilde y_K)] \leq
        16LD\cdot\log^2 (n)\roundBr[\bigg]{\frac{1}{\sqrt{n}} + \frac{5\sqrt{d\log(5/\delta)}}{n\varepsilon}},
    \end{equation*}
    when setting $\mu=(L/D)\max\curlyBr*{2/\sqrt{n},
    13\log(n)\sqrt{d\log(5/\delta)}/(n\varepsilon)}$, where $d=\max\{d_x,d_y\}$.
\end{theorem}

\begin{proof}[Proof Sketch]
    Throughout the proof, we assume that the population function $F(x,y)$ has at least one saddle point $(x^*, y^*)\in\cX\times\cY$. We first give the guarantees of Algorithm \ref{algo:perturb-phased_cc}. By Theorem \ref{thm:scsc}, we know that each phase is $(\varepsilon/2, \delta/2)$-DP and the output $\tilde x_k$ satisfies that
    \begin{equation*}
        \bE\norm{\tilde x_k-\hat x_k^*}^2 \leq \frac{257L^2}{\mu_k}\frac{d_x\log(5/\delta)}{\mu\bar n^2\varepsilon^2}.
    \end{equation*}
    Then by parallel composition, Algorithm \ref{algo:perturb-phased_cc} is $(\varepsilon/2, \delta/2)$-DP. For the utility bound of the output $\tilde x_K$, we start with the following error decomposition:
    \begin{equation*}
        \max_{y\in\cY}\bE[F(\tilde x_K, y)] - F(x^*, y^*) \leq
        \bE[F(\tilde x_K,\hat y_{K+1}^*) -
        F(\hat x_K^*,\hat y_{K+1}^*)] +
        \sum_{k=1}^K \bE \squareBr[\Big]{F(\hat x_k^*, \hat y_{k+1}^*) - 
        F(\hat x_{k-1}^*, \hat y_k^*)},
    \end{equation*}
    where $(\hat x_k^*, \hat y_k^*)$ is the saddle point of the regularized empirical function $\hat F_k(x,y)$ in Algorithm \ref{algo:perturb-phased_cc}, and we let $\hat x_0^*=x^*$ and $\hat y_{K+1}^*\in \arg\max_{y\in\cY} \bE[F(\tilde x_K, y)]$ only for simplicity of the notations in the above telescopic summation. The first term in the RHS of above decomposition is bounded by Lipschitzness of $F(x,y)$, and the second term uses Lemma \ref{lm:gen-algo_cc}. Therefore, we have that
    \begin{align}
        \max_{y\in\cY}\bE[F(\tilde x_K, y)] - F(x^*, y^*)
        & \leq
        L\sqrt{\bE\norm{\tilde x_K - \hat x_K^*}^2} +
        \sum_{k=1}^K \roundBr*{
        \frac{\mu_k}{2}\bE\norm{\hat x_{k-1}^* - \tilde x_{k-1}}^2 + \frac{8L^2}{\mu\bar n}} +
        \frac{\mu}{2}\norm{\hat y_{K+1}^*}^2 \nonumber \\
        & \leq
        8LDK^2\roundBr*{\frac{1}{\sqrt{n}} + \frac{5\sqrt{d\log(5/\delta)}}{n\varepsilon}},
        \label{eq:(thm:cc-primal)}
    \end{align}
    by the guarantees of $\tilde x_k$ and settings of $\mu_k$ and $\mu$. Similarly we can show that Algorithm \ref{algo:perturb-phased_cc-II} is $(\varepsilon/2, \delta/2)$-DP and the output $\tilde y_K$ satisfies that
    \begin{equation*}
        F(x^*, y^*) - \min_{x\in\cX}\bE[F(x, \tilde y_K)] \leq
        8LDK^2\roundBr*{\frac{1}{\sqrt{n}} + \frac{5\sqrt{d\log(5/\delta)}}{n\varepsilon}}.
    \end{equation*}
    Theorem \ref{thm:cc} then follows by summing up the above two bounds. The detailed proof is in Appendix \ref{sec:app-minimax}.
\end{proof}

According to the lower-bound $\Omega\roundBr*{1/\sqrt{n} + \sqrt{d_x\log(1/\delta)}/(n\varepsilon)}$ discussed in \citep{boob2021optimal}, our utility guarantee is optimal up to logarithmic terms. We then analyze the complexity of Algorithm \ref{algo:perturb-phased_cc} and \ref{algo:perturb-phased_cc-II}. Since two algorithms have the same gradient complexity, we only do the computations for Algorithm \ref{algo:perturb-phased_cc}.

The gradient complexity of our algorithm is the total complexity of $\cA$ to solve the smooth SC-SC saddle point problem at each phase. If we regard the regularized empirical problem in Algorithm \ref{algo:perturb-phased_cc} as a general $(\ell+\mu_k)$-smooth, $\mu_k$-strongly convex and $\mu$-strongly concave finite-sum minimax problem, linear time-complexity cannot be achieved since the condition number $(\ell + \mu_k)/\mu \leq \cO(\sqrt{n}\ell D/L) + 2^k$ can be as large as $\cO(n)$. By Remark \ref{rmk:scsc}, the complexity of each phase in Algorithm \ref{algo:perturb-phased_cc} is $\cO(T(\bar n, \kappa_x, \kappa_y)\log(1/\gamma))$ where $\kappa_x=(\ell+\mu_k)/\mu_k$, $\kappa_y=(\ell+\mu_k)/\mu$ and $\gamma=\delta L^2/(16\mu\bar n^2)$ since $\mu\leq\mu_k$. As a result, when using AL-SVRE \citep{luo2021near} and Catalyst-Acc-SVRG \citep{yang2020catalyst}, the total gradient complexity is
\begin{equation*}
    \sum_{k=1}^K \cO\roundBr*{\roundBr*{\bar n +
    \sqrt{\frac{\bar n(\ell+\mu_k)^2}{\mu_k\mu}} +
    {\bar n}^{3/4}\sqrt{\frac{\ell+\mu_k}{\mu}}}\log(1/\gamma)} \leq
    \cO((n + n\sqrt{\ell D/L} + n\ell D/L + n^{5/4})\log(n/\delta)),
\end{equation*}
since $1/\mu\leq\cO(\sqrt{n} D/L)$ and $\sum_{k=1}^K\sqrt{2^k}=\cO(\sqrt{n})$. The complexity of Extragradient \citep{tseng1995linear} can be computed similarly. We find that the super-linear complexity comes from the potentially large smoothness parameter corresponding to the regularizer $(\mu_k/2)\norm{x-\tilde x_{k-1}}^2$. Next we show this can be avoided by algorithms with fine-grained analyses and thus linear time-complexity is possible.

\citet{jin2022sharper} provided sharper rates when the problem is separable. The problem they considered is
\begin{equation*}
    \min_{x\in\cX} \max_{y\in\cY} \; \hat F_1(x,y) \triangleq
    \frac{1}{n} \sum_{i=1}^n f(x,y;\xi_i) + g(x) - h(y),
\end{equation*}
where $f(x,y;\xi)$ is $\ell$-smooth and convex-concave, $g(x)$ is $\ell_g$-smooth, $h(y)$ is $\ell_h$-smooth and the overall function $\hat F_1(x,y)$ is $\mu_x$-strongly convex $\mu_y$-strongly concave. Their algorithms achieve a $\gamma$-approximate saddle point $(\tilde x, \tilde y)$ of $\hat F_1(x,y)$ such that $\bE\squareBr*{\max_{y\in\cY} \hat F_1(\tilde x, y) - \min_{x\in\cX} \hat F_1(x, \tilde y)} \leq \gamma$ with complexity
\begin{equation*}
    \tilde\cO\roundBr*{\roundBr*{n + \sqrt{n}\roundBr*{
    \sqrt{\frac{\ell_g}{\mu_x}} + \sqrt{\frac{\ell_h}{\mu_y}} + \frac{\ell}{\mu_x} + \frac{\ell}{\sqrt{\mu_x\mu_y}} + \frac{\ell}{\mu_y}
    }}\log(1/\gamma)},
\end{equation*}
where $\tilde\cO$ hides additional logarithmic terms. Our regularized problems satisfy the above special structure with $\ell_g=\mu_x=\mu_k$ and $\ell_h=\mu_y=\mu$. This fine-grained analysis is especially suitable for us since the potentially large smoothness parameter $\mu_k$ does not affect the complexity result. The total complexity of Algorithm \ref{algo:perturb-phased_cc} using results in \citet{jin2022sharper} is thus
\begin{equation*}
    \sum_{k=1}^K \tilde\cO\roundBr*{\roundBr*{\bar n + \sqrt{\bar n}\roundBr*{ \frac{\ell}{\mu_k} + \frac{\ell}{\sqrt{\mu_k\mu}} + \frac{\ell}{\mu}
    }}\log(1/\gamma)} \leq
    \tilde\cO((n + n\ell D/L)\log(n/\delta)),
\end{equation*}
which is linear in the number of samples up to logarithmic factors.

The above results suggest that the smoothness parameter of the regularization term in our problems does not affect the gradient complexity. Actually \citet{palaniappan2016stochastic} studied the case
\begin{equation*}
    \min_{x\in\cX} \max_{y\in\cY} \; \hat F_2(x,y) \triangleq
    \frac{1}{n}\sum_{i=1}^n f(x,y;\xi_i) + g(x,y),
\end{equation*}
where $f(x,y;\xi)$ is convex-concave and $\ell$-smooth, and $g(x,y)$ is $\mu_x$-strongly convex, $\mu_y$-strongly concave and possibly \emph{nonsmooth}. Moreover, the nonsmooth part $g(x,y)$ is required to be ``prox-friendly'' in the sense that the proximal operator
\begin{equation*}
    \text{prox}_g^\lambda (x', y') \triangleq \arg \min_{x\in\cX} \max_{y\in\cY} \curlyBr*{\lambda g(x,y) + \frac{\mu_x}{2}\norm{x-x'}^2 - \frac{\mu_y}{2}\norm{y-y'}^2},
\end{equation*}
is easy to compute for any $(x',y')\in\cX\times\cY$. The complexity of their algorithms SVRG/SAGA and Acc-SVRG/SAGA only depends on the smoothness parameter of $f(x,y;\xi)$ and SC-SC parameters $\mu_x$, $\mu_y$, i.e., $\ell/\min\{\mu_x, \mu_y\}$. The algorithms are first designed for matrix games, but also work for general functions (see extensions to monotone operators in their supplementary material \citep{palaniappan2016stochastic}). Our regularized problems satisfy the special structure since the quadratic regularization term is prox-friendly, and thus the condition number $\ell/\mu$ is only $\cO(\sqrt{n}\ell)$. The resulting complexity is
\begin{equation*}
    \sum_{k=1}^K \cO\roundBr*{\roundBr*{
    \bar n + \frac{\sqrt{\bar n}\ell}{\mu}
    }\log(1/\gamma)} = \cO((n + n\ell D/L)\log(n/\delta)),
\end{equation*}
for Acc-SVRG/SAGA and $\cO((n+n(\ell D/L)^2)\log(n/\delta))$ for SVRG/SAGA by a similar computation.

\begin{table}[t]
    \centering
    \caption{Comparisons of the gradient complexity of Algorithm \ref{algo:perturb-scsc} for smooth SC-SC DP-SMO and Algorithm \ref{algo:perturb-phased_cc} and \ref{algo:perturb-phased_cc-II} for smooth convex-concave DP-SMO when equipping with different non-private minimax optimization algorithms. Here $\ell$ is the smoothness parameter of $f(x,y;\xi)$, $\kappa_x=\ell/\mu_x$, $\kappa_y=\ell/\mu_y$, $\kappa=\max\{\kappa_x, \kappa_y\}$ are condition numbers, and $\tilde\cO$ hides logarithmic factors in $n/\delta$.}
    \begin{tabular}{ccc}
        \toprule
        Algorithm & SC-SC & Convex-Concave \\ 
        \midrule
        Extragradient \citep{tseng1995linear}
        & $\tilde\cO(n\kappa)$
        & $\tilde\cO(n^{3/2}\ell D/L+n^2)$ \\
        SVRG/SAGA \citep{palaniappan2016stochastic}
        & $\tilde\cO(n+\kappa^2)$
        & $\tilde\cO(n+n(\ell D/L)^2)$ \\
        Acc-SVRG/SAGA \citep{palaniappan2016stochastic}
        & $\tilde\cO(n+\sqrt{n}\kappa)$
        & $\tilde\cO(n+n\ell D/L)$ \\
        \citet{jin2022sharper}
        & $\tilde\cO(n + \sqrt{n}\kappa)$
        & $\tilde\cO(n+n\ell D/L)$ \\
        AL-SVRE \citep{luo2021near}
        & $\tilde\cO(n + \sqrt{n\kappa_x\kappa_y} + n^{3/4}\sqrt{\kappa})$
        & $\tilde\cO(n\ell D/L+n^{5/4})$ \\
        Catalyst-Acc-SVRG \citep{yang2020catalyst}
        & $\tilde\cO(n + \sqrt{n\kappa_x\kappa_y} + n^{3/4}\sqrt{\kappa})$
        & $\tilde\cO(n\ell D/L+n^{5/4})$ \\
        \bottomrule
    \end{tabular}
    \label{tab:complexity}
\end{table}

\begin{remark}
    The gradient complexity of Algorithm \ref{algo:perturb-phased_cc} and \ref{algo:perturb-phased_cc-II} is $\cO((n^{3/2}\ell D/L+n^2)\log(n/\delta))$ for Extragradient \citep{tseng1995linear}, $\cO((n\ell D/L+n^{5/4})\log(n/\delta))$ for AL-SVRE \citep{luo2021near} and Catalyst-Acc-SVRG \citep{yang2020catalyst}, $\cO((n+n(\ell D/L)^2)\log(n/\delta))$ for SVRG/SAGA \citep{palaniappan2016stochastic}, $\cO((n+n\ell D/L)\log(n/\delta))$ for Acc-SVRG/SAGA \citep{palaniappan2016stochastic} and \citet{jin2022sharper}. See Table \ref{tab:complexity} for a summary.
\end{remark}

Although the primal problem $\min_{x\in\cX}\Phi(x)\triangleq\max_{y\in\cY}F(x,y)$ is not necessarily smooth when $f(x,y;\xi)$ is smooth and convex-concave, we provide several instances of the flexible framework that achieve near-optimal rates in near-linear time. This improves upon current results for DP-SMO \citep{boob2021optimal, yang2022differentially} and gives a new example where near-linear time algorithms are available for nonsmooth DP-SCO in addition to generalized linear losses \citep{bassily2021differentially}. Our results for DP-SMO suggest that for nonsmooth DP-SCO problems, if there exists some smooth convex-concave minimax reformulation, then near-optimal rates can be attained in near-linear time. For example, the nonsmooth convex problem $\min_{x\in\bR^{d_x}}\norm{Ax-b}_1$ given $A\in\bR^{d_y\times d_x}$ and $b\in\bR^{d_y}$ is equivalent to the smooth convex-concave minimax problem $\min_{x\in\bR^{d_x}}\max_{y\in\bR^{d_y}, \norm{y}_\infty\leq 1} y^\top(Ax-b)$. This provides a better understanding of nonsmooth DP-SCO. However, it remains an open question whether optimal utility bound can be obtained in linear time for general nonsmooth problems.

\subsection{Near-Linear Time Algorithms for Convex--Strongly-Concave Functions}

At last, we briefly discuss the smooth convex--strongly-concave (C-SC) and strongly-convex--concave (SC-C) DP-SMO problems as the special case of convex-concave settings. Actually Algorithm \ref{algo:perturb-phased_cc} can be extended to the C-SC setting, and Algorithm \ref{algo:perturb-phased_cc-II} can be extended to the SC-C case. We will focus on the C-SC case in this section. The SC-C setting is similar and omitted here.

We consider the DP-SMO problem when the objective function $f(x,y;\xi)$ satisfies Assumption \ref{asp:cc}, and additionally $f(x,\cdot;\xi)$ is $\mu_y$-strongly concave on $\cY$ for any $x\in\cX$. When $\mu_y\leq\cO(1/\sqrt{n})$, the strong-concavity does not help too much. We can directly regard it as a convex-concave problem and apply Algorithm \ref{algo:perturb-phased_cc} and \ref{algo:perturb-phased_cc-II} to solve it. However, there is no need to add the regularization term for $y$ in Algorithm \ref{algo:perturb-phased_cc} when $\mu_y\geq\cO(1/\sqrt{n})$. The resulting variant for smooth C-SC DP-SMO iteratively solves
\begin{equation*}
    \min_{x\in\cX} \max_{y\in\cY} \; \curlyBr*{
    \frac{1}{\bar n}\sum_{i=(k-1)\bar n+1}^{k\bar n} f(x,y;\xi_i) + \frac{\mu_k}{2}\norm{x-\tilde x_{k-1}}^2},
\end{equation*}
at each phase $k$ to obtain the output $(x_k, y_k)$ such that with probability $1-\delta/4$,
\begin{equation*}
    \mu_k \norm{x_k-\hat x_k^*}^2 + \mu_y \norm{y_k-\hat y_k^*}^2 \leq \frac{L^2}{\bar n^2 \min\{\mu_k, \mu_y\}},
\end{equation*}
where $(\hat x_k^*, \hat y_k^*)$ is the saddle point of the regularized empirical function. Then the perturbed output $\tilde x_k$ is obtained by adding Gaussian noise $\cN(0, \sigma_k^2\rI_{d_x})$ to $x_k$ with variance
$\sigma_k = (4L/(\bar n\varepsilon))\sqrt{2\log(2.5/\delta)/(\mu_k\min\{\mu_k, \mu_y\})}$. As a direct corollary of Theorem \ref{thm:cc}, the final output $\tilde x_K$ satisfies the following guarantees.
\begin{corollary}
    Let Assumption \ref{asp:cc} hold. Assume that $f(x,\cdot;\xi)$ is $\mu_y$-strongly concave on $\cY$ for any $x\in\cX$ with $\mu_y\geq L/(D\sqrt{n})$. Suppose $\max\{\norm{x},\norm{y}\}\leq D$ for all $x\in\cX$ and $y\in\cY$. Then the variant of Algorithm \ref{algo:perturb-phased_cc} discussed above is $(\varepsilon, \delta)$-DP and its output $\tilde x_K$ satisfies that
    \begin{equation*}
        \max_{y\in\cY} \bE[F(\tilde x_K, y)] - F(x^*, y^*)
        \leq
        6LD\log(n)\roundBr*{\frac{1}{\sqrt{n}} + \frac{4\sqrt{d_x\log(2.5/\delta)}}{n\varepsilon}} +
        L^2\log^2(n)\roundBr*{\frac{8}{\mu_y n} + \frac{65d_x\log(2.5n/\delta)}{\mu_y n^2\varepsilon^2}},
    \end{equation*}
    when setting $\mu=3(L/D)\max\curlyBr*{1/\sqrt{n},
    4\log(n)\sqrt{d_x\log(2.5/\delta)}/(n\varepsilon)}$.
\end{corollary}
The proof directly follows from \eqref{eq:(thm:cc-primal)} when replacing $\mu$ by $\min\{\mu_k, \mu_y\}$ for every $\mu$ appearing in the denominator and the fact that $1/\min\{\mu_k, \mu_y\}\leq 1/\mu_k+1/\mu_y$. The result also recovers Theorem \ref{thm:cc} if $\mu_y\geq\mu$. When $f(x,y;\xi)$ is convex--strongly-concave, the primal function $\Phi(x)=\max_{y\in\cY} F(x,y)$ is smooth and convex. Since $F(x^*, y^*)=\Phi(x^*)$, the above guarantees imply that
\begin{equation*}
    \bE[\Phi(\tilde x_K) - \Phi(x^*)]
    \leq
    6LD\log(n)\roundBr*{\frac{1}{\sqrt{n}} + \frac{4\sqrt{d_x\log(2.5/\delta)}}{n\varepsilon}} +
    L^2\log^2(n)\roundBr*{\frac{8}{\mu_y n} + \frac{65d_x\log(2.5n/\delta)}{\mu_y n^2\varepsilon^2}},
\end{equation*}
which is the optimal rate, up to logarithmic factors, for the smooth convex DP-SCO problem $\min_{x\in\cX} \Phi(x)$. Then we analyze the gradient complexity. By similar computations as the above section, the total complexity is $\tilde\cO(n+n\ell^2+\kappa_y^2)$ for SVRG/SAGA \citep{palaniappan2016stochastic}, $\tilde\cO(n+n\ell+\sqrt{n}\kappa_y)$ for Acc-SVRG/SAGA \citep{palaniappan2016stochastic} and $\tilde\cO(n+n\ell+\sqrt{n}\kappa_y)$ using \citet{jin2022sharper}, where $\tilde\cO$ hides additional logarithmic terms and $\kappa_y=\ell/\mu_y$. Not surprisingly, it is possible to achieve near-optimal rates with near-linear time-complexity.

\section{Conclusion} \label{sec:conclude}

We provide a general framework for both smooth DP-SCO and DP-SMO problems with the near-optimal privacy-utility trade-offs. The flexible framework allows to bring various off-the-shelf, fast convergent non-private optimization algorithms into the DP domain. Using the framework, we enrich the class of near-linear time algorithms for smooth DP-SCO and provide the first near-linear time algorithms for smooth DP-SMO. For future work, it is interesting to study whether the logarithmic terms in gradient complexity and the final utility bound of our algorithms can be further removed.  Another direction is to derive simpler private algorithms for smooth convex-concave DP-SMO based on the special structure of the regularized problems.
Moreover, we believe our framework also opens the door to simpler derivatives with possibly improved complexities for nonsmooth DP-SCO, nonsmooth DP-SMO, and even more general forms of optimization problems such as variational inequalities and games. We will leave these for future investigation.

\section*{Acknowledgements}{
    L.Z. gratefully acknowledges funding by the Max Planck ETH Center for Learning Systems (CLS).
    This work does not relate to the current position of K.T. at Amazon. 
    S.O. is  supported in part by NSF grants CNS-2002664, IIS-1929955, and CCF-2019844 as a part of NSF Institute for Foundations of Machine Learning (IFML).
    N.H. is supported by ETH research grant funded through ETH Zurich Foundations and NCCR Automation funded through Swiss National Science Foundations; part of the work was done while N.H. was visiting the Simons Institute for the Theory of Computing.
}

\bibliographystyle{plainnat}
\bibliography{ref.bib}

\clearpage
\appendix

\section{Differentially Private Stochastic Convex Optimization} \label{sec:app-min}

This section contains analyses of algorithms for DP-SCO. We first present some helpful lemmas.

\subsection{Supporting Lemmas}

The lemma below provides a specific form of parallel composition in Lemma \ref{lm:composition} that can be directly applied to show $(\varepsilon, \delta)$-DP of the phased algorithms.

\begin{lemma}
    (Parallel Composition \citep{mcsherry2009privacy})
    Given a dataset $S$ and its disjoint partition $S=\bigcup_{k=1}^K S_k$, define the mechanisms as $\cA_1=\cA_1(S_1), \cA_2=\cA_2(S_2; \cA_1), \cdots, \cA_K=\cA_K(S_K; \cA_{K-1})$. Suppose each mechanism $\cA_k(S_k; \cA_{k-1})$ is $(\varepsilon, \delta)$-DP with respect to the set $S_k$ for $k=1,\cdots,K$, then the composition $\cA_K$ is $(\varepsilon, \delta)$-DP with respect to the full dataset $S$.
    \label{lm:parallel}
\end{lemma}

\begin{proof}
    For neighboring datasets $S\sim S'$, without loss of generality we let $S=\{\xi_1,\cdots,\xi_i,\cdots,\xi_n\}$ and $S'=\{\xi_1,\cdots,\xi_i',\cdots,\xi_n\}$, where $\xi_i$ and $\xi_i'$ are sampled independently. That is, the only difference between the datasets comes from $\xi_i$ and $\xi_i'$, and the remaining samples are the same. As a result, for the disjoint partitions $S=\bigcup_{k=1}^K S_k$ and $S'=\bigcup_{k=1}^K S_k'$, we can conclude that there is only one pair $S_j\sim S_j'$ for some $j\in\{1,\cdots,K\}$, and that $S_k=S_k'$ for $k\neq j$. Then we have that
    \begin{align*}
        \bP(\cA_K(S))
        & = \prod_{k=1}^K \bP(\cA_k(S_k)|\cA_{k-1}) \\
        & = \bP(\cA_j(S_j)|\cA_{j-1}) \prod_{k=1, k\neq j}^K \bP(\cA_k(S_k')|\cA_{k-1}) \\
        & \leq \roundBr[\Big]{e^\varepsilon\bP(\cA_j(S_j')|\cA_{j-1}) + \delta}
        \prod_{k=1, k\neq j}^K \bP(\cA_k(S_k')|\cA_{k-1}) \\
        & \leq e^\varepsilon\bP(\cA_K(S')) + \delta.
    \end{align*}
    By Definition \ref{def:dp} of $(\varepsilon, \delta)$-differential privacy, the proof is complete.
\end{proof}

For the stability and generalization results of the empirical risk minimization \citep{shalev2009stochastic}, we give a detailed proof of the regularized version in Lemma \ref{lm:gen-c} for completeness. The proof of Lemma \ref{lm:gen-sc} can be derived similarly. Note that all the stability results depend on the Lipschitzness parameter of the objectives, making $L$ a critical parameter in private algorithms \citep{bassily2019private, feldman2020private, yang2022differentially}. In practice, any estimate of the upper bound of Lipschitz constant $L$ can be used, e.g., see methods in \citep{wood1996estimation, fazlyab2019efficient}.

\begin{proof}[Proof of Lemma \ref{lm:gen-c}]
    Without loss of generality, for any $1\leq i\leq n$, we let neighboring datasets $S=\{\xi_1,\cdots,\xi_i,\cdots,\xi_n\}$ and $S_i'=\{\xi_1,\cdots,\xi_i',\cdots,\xi_n\}$, where $\xi_i$ and $\xi_i'$ are sampled independently.
    
    Since $\hat F_S(x) + G(x)$ is $\mu$-strongly convex and $\hat x_S^*$ is the optimal solution, we have that
    \begin{equation*}
        \hat F_S(\hat x_{S_i'}^*) + G(\hat x_{S_i'}^*)
        \geq \hat F_S(\hat x_S^*) + G(\hat x_S^*) +
        \frac{\mu}{2}\norm{\hat x_S^* - \hat x_{S_i'}^*}^2.
    \end{equation*}
    Similarly, by strong-convexity of $\hat F_{S_i'} + G(x)$ and optimality of $\hat x_{S_i'}^*$, we have that
    \begin{equation*}
        \hat F_{S_i'}(\hat x_S^*) + G(\hat x_S^*)
        \geq \hat F_{S_i'}(\hat x_{S_i'}^*) + G(\hat x_{S_i'}^*) +
        \frac{\mu}{2}\norm{\hat x_S^* - \hat x_{S_i'}^*}^2.
    \end{equation*}
    Summing up the above two equations, we can obtain
    \begin{align}
        \mu\norm{\hat x_S^*-\hat x_{S_i'}^*}^2
        & \leq
        \hat F_S(\hat x_{S_i'}^*) - \hat F_{S_i'}(\hat x_{S_i'}^*) +
        \hat F_{S_i'}(\hat x_S^*) - \hat F_S(\hat x_S^*) \nonumber \\
        & =
        \frac{1}{n}[f(\hat x_{S_i'}^*;\xi_i) - f(\hat x_{S_i'}^*;\xi_i')] +
        \frac{1}{n}[f(\hat x_S^*;\xi_i') - f(\hat x_S^*;\xi_i)] \nonumber \\
        & =
        \frac{1}{n}[f(\hat x_{S_i'}^*;\xi_i) - f(\hat x_S^*;\xi_i)] +
        \frac{1}{n}[f(\hat x_S^*;\xi_i') - f(\hat x_{S_i'}^*;\xi_i')] \nonumber \\
        & \leq
        \frac{2L}{n}\norm{\hat x_S^* - \hat x_{S_i'}^*},
        \label{eq-pf:(lm:gen-c)-(a)}
    \end{align}
    where the first equality holds since for any $x$,
    \begin{align}
        \hat F_S(x) - \hat F_{S_i'}(x)
        & =
        \frac{1}{n}\sum_{j=1}^n f(x;\xi_j) -
        \frac{1}{n}\roundBr*{\sum_{j=1, j\neq i}^n f(x;\xi_j) + f(x;\xi_i')}
        \nonumber \\
        & =
        \frac{1}{n}[f(x;\xi_i) - f(x;\xi_i')],
        \label{eq-pf:(lm:gen-c)-(fs)}
    \end{align}
    and the last inequality follows from $L$-Lipschitzness of $f(x;\xi)$. As a result of \eqref{eq-pf:(lm:gen-c)-(a)}, we obtain the stability of empirical solutions as
    \begin{equation}
        \norm{\hat x_S^* - \hat x_{S_i'}^*} \leq \frac{2L}{\mu n}.
        \label{eq-pf:(lm:gen-c)-(stability)}
    \end{equation}
    For the generalization error, we follow the standard results on stability and generalization. Let $x^*=\arg\min_{x\in\cX}\{F(x)+G(x)\}$ for notation simplicity, and then
    \begin{align*}
        \bE[F(\hat x_S^*) + G(\hat x_S^*)] - \bE[F(x^*) + G(x^*)]
        & \overset{(a)}{=}
        \bE[F(\hat x_S^*) + G(\hat x_S^*)] - \bE[\hat F_S(x^*) + G(x^*)] \\
        & \overset{(b)}{\leq}
        \bE[F(\hat x_S^*) + G(\hat x_S^*)] -
        \bE[\hat F_S(\hat x_S^*) + G(\hat x_S^*)] \\
        & \overset{(c)}{=}
        \bE\squareBr*{\frac{1}{n}\sum_{i=1}^n F(\hat x_{S_i'}^*) -
        \frac{1}{n}\sum_{i=1}^n f(\hat x_S^*;\xi_i)} \\
        & \overset{(d)}{=}
        \frac{1}{n} \sum_{i=1}^n
        \bE[f(\hat x_{S_i'}^*;\xi_i) - f(\hat x_S^*;\xi_i)] \\
        & \overset{(e)}{\leq}
        \frac{2L^2}{\mu n},
    \end{align*}
    where $(a)$ holds since $x^*$ is independent of $S$, $(b)$ follows by the optimality of $\hat x_S^*$, $(c)$ uses the fact that $\hat x_S^*$ and $\hat x_{S_i'}^*$ have the same distribution for each $i$, $(d)$ is true because $S_i'$ is independent of $\xi_i$ and $(e)$ uses Lipschitzness of $f(x;\xi)$ and stability bound \eqref{eq-pf:(lm:gen-c)-(stability)}.
\end{proof}

\subsection{Near-Linear Time Algorithms for Smooth Strongly-Convex Functions}

The proof of Theorem \ref{thm:excess-sc} that gives the guarantees of Algorithm \ref{algo:perturb-sc} is provided below.

\begin{proof}[Proof of Theorem \ref{thm:excess-sc}]
    We first prove the privacy guarantee. Given neighboring datasets $S\sim S'$, the sensitivity of $\cA$ in Algorithm \ref{algo:perturb-sc} is bounded as
    \begin{align}
        \norm{\cA(S) - \cA(S')}
        & \leq
        \norm{\cA(S)-\hat x_S^*} + \norm{\hat x_S^*-\hat x_{S'}^*} +
        \norm{\hat x_{S'}^* - \cA(S')} \nonumber \\
        & \leq \frac{4L}{\mu n},
        \label{eq-pf:(thm:excess-sc)-(sensitivity)}
    \end{align}
    with probability $1-\delta/2$ by the union bound, where the last inequality follows from the stability of empirical solutions in Lemma \ref{lm:gen-sc} and guarantees of algorithm $\cA$. Then by Gaussian mechanism in Definition \ref{def:gauss}, Algorithm \ref{algo:perturb-sc} is $(\varepsilon, \delta)$-DP when setting $\sigma=4L\sqrt{2\log(2.5/\delta)}/(\mu n\varepsilon)$.

    We then give the guarantees for the output $\tilde x$. Since $\hat F_S$ is $\ell$-smooth and $\hat x_S^*$ is in the interior of $\cX$, the excess empirical risk satisfies that
    \begin{align}
        \bE[\hat F_S(\tilde x) - \hat F_S(\hat x_S^*)]
        & \leq
        \frac{\ell}{2}\bE\norm{\tilde x - \hat x_S^*}^2 \nonumber \\
        & \leq
        \ell\roundBr*{\bE\norm{\tilde x - \cA(S)}^2 +
        \bE\norm{\cA(S) - \hat x_S^*}^2} \nonumber \\
        & \leq
        \ell\roundBr*{
        \frac{32L^2\cdot d\log(2.5/\delta)}{\mu^2 n^2\varepsilon^2} + \frac{\delta^2 L^2}{16 \mu^2 n^2}} \nonumber \\
        & <
        33L^2\kappa\cdot\frac{d\log(2.5/\delta)}{\mu n^2\varepsilon^2},
        \label{eq-pf:(thm:excess-sc)-(utility)}
    \end{align}
    where the third inequality is due to the guarantee of $\cA$ in Remark \ref{rmk:sc} and the choice of $\sigma$, and the last inequality follows from the standard settings that $\varepsilon<1$, $d\geq 1$ and $\delta<1/n$, and $\kappa=\ell/\mu$ is the condition number. Similarly for the excess population risk, we have that
    \begin{align*}
        \bE[F(\tilde x) - F(x^*)]
        & \leq
        \frac{\ell}{2} \bE\norm{\tilde x - x^*}^2 \\
        & \leq
        \frac{3\ell}{2}\roundBr*{\bE\norm{\tilde x - \cA(S)}^2 +
        \bE\norm{\cA(S) - \hat x_S^*}^2 + \bE\norm{\hat x_S^* - x^*}^2} \\
        & \leq
        \ell \roundBr*{
        \frac{48L^2\cdot d\log(2.5/\delta)}{\mu^2 n^2\varepsilon^2} + \frac{3\delta^2L^2}{32\mu^2 n^2} + \frac{6L^2}{\mu^2 n}} \\
        & <
        L^2\kappa\roundBr*{\frac{7}{\mu n} +
        \frac{48d\log(2.5/\delta)}{\mu n^2\varepsilon^2}},
    \end{align*}
    where we use the fact that $(\mu/2)\bE\norm{\hat x_S^* - x^*}^2 \leq \bE[F(\hat x_S^*) - F(x^*)]$ by strong-convexity and the generalization bound in Lemma \ref{lm:gen-sc}. Note that we can still obtain the bound for $\bE\norm{\tilde x - \hat x_S^*}^2$ and $\bE\norm{\tilde x - x^*}^2$ if we do not assume $\hat x_S^*$ and $x^*$ are interior points. 
\end{proof}

\subsection{Near-Linear Time Algorithms for Smooth Convex Functions}

This section contains the proof of Theorem \ref{thm:excess-c} that gives the guarantees of Algorithm \ref{algo:perturb-phased_c} for convex functions.

\begin{proof}[Proof of Theorem \ref{thm:excess-c}]
    We first give the guarantees of each phase in Algorithm \ref{algo:perturb-phased_c} by a similar proof as Theorem \ref{thm:excess-sc}, i.e., see \eqref{eq-pf:(thm:excess-sc)-(sensitivity)} and \eqref{eq-pf:(thm:excess-sc)-(utility)}. For phase $1\leq k\leq K$, by the stability of $\hat x_k^*$ in Lemma \ref{lm:gen-c} and the requirement of output $x_k$ in Algorithm \ref{algo:perturb-phased_c}, we know $x_k$ has sensitivity bounded by $4L/(\mu_k\bar n)$ with probability at least $1-\delta/2$, and then setting
    $\sigma_k=4L\sqrt{2\log(2.5/\delta)}/(\mu_k\bar n\varepsilon)$ guarantees $(\varepsilon, \delta)$-DP. As a result, we can obtain that
    \begin{align}
        \bE\norm{\tilde x_k - \hat x_k^*}^2
        & \leq
        2\bE\norm{\tilde x_k - x_k}^2 + 2\bE\norm{x_k - \hat x_k^*}^2 \nonumber\\
        & \leq
        2d\sigma_k^2 + \frac{\delta^2L^2}{8\mu_k^2\bar n^2} \nonumber \\
        & <
        65L^2\cdot\frac{d\log(2.5/\delta)}{\mu_k^2 \bar n^2\varepsilon^2}.
        \label{eq-pf:(thm:excess-c)-(phase)}
    \end{align}
    
    Then we analyze the full algorithm. By the parallel composition in Lemma \ref{lm:parallel}, Algorithm \ref{algo:perturb-phased_c} is $(\varepsilon, \delta)$-DP since we use disjoint datasets for different phases and each phase is $(\varepsilon, \delta)$-DP. For the excess population risk of the output $\tilde x_K$, we decompose the error as
    \begin{equation}
        \bE[F(\tilde x_K) - F(x^*)]
        = \bE[F(\tilde x_K) - F(\hat x_K^*)] +
        \sum_{k=1}^K \bE[F(\hat x_k^*) - F(\hat x_{k-1}^*)],
        \label{eq-pf:(thm:excess-c)-(decompose)}
    \end{equation}
    where $x^*\in\arg\min_{x\in\cX} F(x)$ is the population optimal solution, $\hat x_k^*$ is the optimal solution of the regularized empirical function $\hat F_k(x)$ in Algorithm \ref{algo:perturb-phased_c}, and we let
    $\hat x_0^*=x^*$ only for simplicity of the analysis. For the first term in the RHS of \eqref{eq-pf:(thm:excess-c)-(decompose)}, we have that
    \begin{align*}
        \bE[F(\tilde x_K) - F(\hat x_K^*)]
        & \overset{(a)}{\leq}
        L\sqrt{\bE\norm{\tilde x_K - \hat x_K^*}^2} \\
        & \overset{(b)}{<}
        9L^2\cdot
        \frac{\sqrt{d\log(2.5/\delta)}}{\mu_K\bar n\varepsilon} \\
        & \overset{(c)}{\leq}
        9LD\cdot\frac{\sqrt{d\log(2.5/\delta)}}{\bar n\varepsilon},
    \end{align*}
    where $(a)$ holds by $L$-Lipschitzness of $F(x)$ and Cauchy–Schwarz inequality, $(b)$ uses \eqref{eq-pf:(thm:excess-c)-(phase)} and $(c)$ follows from the settings that $\mu_K=\mu n$ and $\mu\geq L/(D\sqrt{n})$. Therefore, with Lemma \ref{lm:gen-algo_c} to handle the second term in the RHS of \eqref{eq-pf:(thm:excess-c)-(decompose)}, we obtain that
    \begin{align*}
        \bE[F(\tilde x_K) - F(x^*)]
        & \leq
        9LD\cdot\frac{\sqrt{d\log(2.5/\delta)}}{\bar n\varepsilon} +
        \sum_{k=1}^K \roundBr*{
        \frac{\mu_k}{2} \bE\norm{\hat x_{k-1}^* - \tilde x_{k-1}}^2 + \frac{2L^2}{\mu_k\bar n}} \\
        & \leq
        9LD\cdot\frac{\sqrt{d\log(2.5/\delta)}}{\bar n\varepsilon} +
        \sum_{k=2}^K \frac{65\mu_k}{2\mu_{k-1}}
        \frac{L^2 d\log(2.5/\delta)}{\mu_{k-1} \bar n^2\varepsilon^2} +
        \sum_{k=1}^K\frac{2L^2}{\mu_k\bar n} + \frac{\mu_1}{2}\norm{x^*-x_0}^2 \\
        & \leq
        9LD\cdot\frac{\sqrt{d\log(2.5/\delta)}}{\bar n\varepsilon} +
        \frac{65L^2 d\log(2.5/\delta)}{\mu \bar n^2\varepsilon^2} +
        \frac{2L^2}{\mu\bar n} + \mu\norm{x^*-x_0}^2 \\
        & \leq
        4LD\cdot\log(n)\roundBr*{\frac{1}{\sqrt{n}} + \frac{7\sqrt{d\log(2.5/\delta)}}{n\varepsilon}},
    \end{align*}
    where the second inequality uses the guarantees of $\tilde x_{k-1}$ in \eqref{eq-pf:(thm:excess-c)-(phase)} for $k\geq 2$ and the settings that $\hat x_0^*=x^*, \tilde x_0=x_0$, the third inequality follows from the choice that $\mu_k=\mu\cdot2^k$ and the last inequality holds since
    $\mu=(L/D)\max\{1/\sqrt{n}, 14\log(n)\sqrt{d\log(2.5/\delta)}/(n\varepsilon)\}$
    and $\norm{x^*}^2\leq D^2$ when the initialization is $x_0=0$.
\end{proof}

\section{Differentially Private Stochastic Minimax Optimization} \label{sec:app-minimax}

In this section, we give the analysis of algorithms for smooth DP-SMO. We first present some useful lemmas.

\subsection{Supporting Lemmas}

Here we overuse notations to provide some general properties of the duality gap for a smooth and strongly-convex--strongly-concave function $f(x,y)$.

\begin{lemma}
    Let $f(x,y)$ be a $\mu_x$-strongly convex $\mu_y$-strongly concave function with one saddle point $(x^*, y^*)\in\cX\times\cY$. Then for any $(\tilde x, \tilde y)\in\cX\times\cY$, its duality gap satisfies that
    \begin{equation*}
        \max_{y\in\cY} f(\tilde x, y) - \min_{x\in\cX} f(x, \tilde y) \geq \frac{\mu_x}{2}\norm{\tilde x-x^*}^2 +
        \frac{\mu_y}{2}\norm{\tilde y-y^*}^2.
    \end{equation*}
    If $f(x,y)$ is also $\ell$-smooth and the saddle point $(x^*, y^*)$ is in the interior of $\cX\times\cY$, it holds that
    \begin{equation*}
        \max_{y\in\cY} f(\tilde x, y) - \min_{x\in\cX} f(x, \tilde y) \leq
        \frac{(\kappa_y+1)\ell}{2}\norm{\tilde x-x^*}^2 +
        \frac{(\kappa_x+1)\ell}{2}\norm{\tilde y-y^*}^2,
    \end{equation*}
    where $\kappa_x=\ell/\mu_x$ and $\kappa_y=\ell/\mu_y$ are condition numbers.
    \label{lm:gap-minimax}
\end{lemma}

\begin{proof}
    First, by optimality of $(x^*, y^*)$ and strong-convexity of $f(\cdot, y^*)$ and $-f(x^*, \cdot)$, we know
    \begin{align*}
        \max_{y\in\cY} f(\tilde x, y) - \min_{x\in\cX} f(x, \tilde y)
        & \geq
        f(\tilde x, y^*) - f(x^*, y^*) + f(x^*, y^*) - f(x^*, \tilde y) \\
        & \geq
        \frac{\mu_x}{2}\norm{\tilde x - x^*}^2 +
        \frac{\mu_y}{2}\norm{\tilde y - y^*}^2.
    \end{align*}
    The inequality also implies that for any $(\tilde x, \tilde y)\in\cX\times\cY$, it holds that
    \begin{equation}
        f(\tilde x, y^*) - f(x^*, \tilde y) \geq
        \frac{\mu_x}{2}\norm{\tilde x - x^*}^2 +
        \frac{\mu_y}{2}\norm{\tilde y - y^*}^2.
        \label{eq-pf:(lm:gap-minimax)-(geq)}
    \end{equation}
    
    To show the second part, we first introduce some notations. Let $\Phi(x)=\max_{y\in\cY} f(x,y)$ be the primal function and $\Psi(y)=\min_{x\in\cX} f(x,y)$ be the dual function. When $f(x,y)$ is $\mu_x$-strongly convex, $\mu_y$-strongly concave and $\ell$-smooth, we have that $\Phi(x)$ is $(\kappa_y+1)\ell$-smooth and $\Psi(y)$ is $(\kappa_x+1)\ell$-smooth, which is a standard result in the literature (e.g., see \citep[Proposition 1]{zhang2021generalization}). Since $(x^*, y^*)$ is the interior point, we know that
    \begin{align*}
        \max_{y\in\cY} f(\tilde x, y) - \min_{x\in\cX} f(x, \tilde y)
        & =
        \Phi(\tilde x) - f(x^*, y^*) + f(x^*, y^*) - \Psi(\tilde y) \\
        & =
        \Phi(\tilde x) - \Phi(x^*) + \Psi(y^*) - \Psi(\tilde y) \\
        & \leq
        \frac{(\kappa_y+1)\ell}{2}\norm{\tilde x - x^*}^2 +
        \frac{(\kappa_x+1)\ell}{2}\norm{\tilde y - y^*}^2,
    \end{align*}
    where the last inequality follows by the smoothness of $\Phi(x)$ and $-\Psi(y)$. The same result is also obtained by \citet{zhang2021generalization} (see proof of Theorem 3 in their appendix).
\end{proof}

For the stability and generalization results of empirical saddle point problems \citep{zhang2021generalization}, we only prove the regularized version in Lemma \ref{lm:gen-cc} for completeness. The proof of Lemma \ref{lm:gen-scsc} is simpler and nearly the same.

\begin{proof}[Proof of Lemma \ref{lm:gen-cc}]
    The same as the proof of Lemma \ref{lm:gen-c}, we let neighboring datasets $S=\{\xi_1,\cdots,\xi_i,\cdots,\xi_n\}$ and $S_i'=\{\xi_1,\cdots,\xi_i',\cdots,\xi_n\}$, where $\xi_i$ and $\xi_i'$ are sampled independently.
    
    Applying \eqref{eq-pf:(lm:gap-minimax)-(geq)} in the proof of Lemma \ref{lm:gap-minimax} for $\hat F_S(x,y) + G(x,y)$ and
    $\hat F_{S_i'}(x,y) + G(x,y)$, we obtain
    \begin{align*}
        &
        \hat F_S(\hat x_{S_i'}^*, \hat y_S^*) +
        G(\hat x_{S_i'}^*, \hat y_S^*) -
        \hat F_S(\hat x_S^*, \hat y_{S_i'}^*) -
        G(\hat x_S^*, \hat y_{S_i'}^*) \geq
        \frac{\mu_x}{2}\norm{\hat x_S^* - \hat x_{S_i'}^*}^2 +
        \frac{\mu_y}{2}\norm{\hat y_S^* - \hat y_{S_i'}^*}^2, \\
        &
        \hat F_{S_i'}(\hat x_S^*, \hat y_{S_i'}^*) +
        G(\hat x_S^*, \hat y_{S_i'}^*) -
        \hat F_{S_i'}(\hat x_{S_i'}^*, \hat y_S^*) -
        G(\hat x_{S_i'}^*, \hat y_S^*) \geq
        \frac{\mu_x}{2}\norm{\hat x_S^* - \hat x_{S_i'}^*}^2 +
        \frac{\mu_y}{2}\norm{\hat y_S^* - \hat y_{S_i'}^*}^2.
    \end{align*}
    Summing up the above two inequalities, we can get
    \begin{align}
        &
        \mu_x\norm{\hat x_S^* - \hat x_{S_i'}^*}^2 +
        \mu_y\norm{\hat y_S^* - \hat y_{S_i'}^*}^2 \nonumber \\
        & \qquad \overset{\phantom{(a)}}{\leq}
        \hat F_S(\hat x_{S_i'}^*, \hat y_S^*) -
        \hat F_{S_i'}(\hat x_{S_i'}^*, \hat y_S^*) +
        \hat F_{S_i'}(\hat x_S^*, \hat y_{S_i'}^*) -
        \hat F_S(\hat x_S^*, \hat y_{S_i'}^*) \nonumber \\
        & \qquad\overset{(a)}{=}
        \frac{1}{n}[f(\hat x_{S_i'}^*, \hat y_S^*;\xi_i) -
        f(\hat x_{S_i'}^*, \hat y_S^*;\xi_i')] +
        \frac{1}{n}[f(\hat x_S^*, \hat y_{S_i'}^*;\xi_i') -
        f(\hat x_S^*, \hat y_{S_i'}^*;\xi_i)] \nonumber \\
        & \qquad\overset{\phantom{(a)}}{=}
        \frac{1}{n}[f(\hat x_{S_i'}^*, \hat y_S^*;\xi_i) -
        f(\hat x_S^*, \hat y_{S_i'}^*;\xi_i)] +
        \frac{1}{n}[f(\hat x_S^*, \hat y_{S_i'}^*;\xi_i') -
        f(\hat x_{S_i'}^*, \hat y_S^*;\xi_i')] \nonumber \\
        & \qquad\overset{(b)}{\leq}
        \frac{2L}{n}\sqrt{\norm{\hat x_S^* - \hat x_{S_i'}^*}^2 +
        \norm{\hat y_S^* - \hat y_{S_i'}^*}^2} \nonumber \\
        & \qquad\overset{(c)}{\leq}
        \frac{2L}{n\sqrt{\mu}}\cdot
        \sqrt{\mu_x\norm{\hat x_S^* - \hat x_{S_i'}^*}^2 +
        \mu_y\norm{\hat y_S^* - \hat y_{S_i'}^*}^2}.
        \label{eq-pf:(lm:gen-cc)-(a)}
    \end{align}
    where $(a)$ uses the following equation that holds for any $(x,y)$:
    \begin{equation*}
        \hat F_S(x,y)-\hat F_{S_i'}(x,y) = \frac{1}{n}[f(x,y;\xi_i)-f(x,y;\xi_i')], 
    \end{equation*}
    as a consequence of \eqref{eq-pf:(lm:gen-c)-(fs)}, $(b)$ is true since $f(x,y;\xi)$ is $L$-Lipschitz, i.e.,
    \begin{equation*}
        \abs{f(x_1,y_1;\xi)-f(x_2,y_2;\xi)}^2 \leq L^2\roundBr*{
        \norm{x_1-x_2}^2 + \norm{y_1-y_2}^2}, \quad
        \forall x_1, x_2\in\cX, y_1,y_2\in\cY,
    \end{equation*}
    and $(c)$ follows from the definition that $\mu=\min\{\mu_x,\mu_y\}$. As a result of \eqref{eq-pf:(lm:gen-cc)-(a)}, we have the following stability bound that holds for neighboring datasets $S\sim S_i'$:
    \begin{equation}
        \mu_x\norm{\hat x_S^* - \hat x_{S_i'}^*}^2 +
        \mu_y\norm{\hat y_S^* - \hat y_{S_i'}^*}^2 \leq \frac{4L^2}{\mu n^2}.
        \label{eq-pf:(lm:gen-cc)-(stability)}
    \end{equation}
    
    Then we can analyze the generalization error using stability results. We first bound the error of $\hat x_S^*$ as:
    \begin{align}
        &
        \max_{y\in\cY}\bE\squareBr[\Big]{F(\hat x_S^*, y) + G(\hat x_S^*, y)} -
        \max_{y\in\cY}\bE\squareBr*{\hat F_S(\hat x_S^*, y) + G(\hat x_S^*, y)} \nonumber \\
        & \qquad \overset{(a)}{\leq}
        \max_{y\in\cY} \curlyBr*{
        \bE[F(\hat x_S^*, y)] - \bE[\hat F_S(\hat x_S^*, y)]} \nonumber \\
        & \qquad \overset{(b)}{=}
        \max_{y\in\cY} \curlyBr*{\frac{1}{n} \sum_{i=1}^n
        \bE[F(\hat x_{S_i'}^*, y)] - \bE[\hat F_S(\hat x_S^*, y)]} \nonumber \\
        & \qquad \overset{(c)}{=}
        \max_{y\in\cY}\frac{1}{n}\sum_{i=1}^n\bE\squareBr*{
        f(\hat x_{S_i'}^*, y; \xi_i) - f(\hat x_S^*, y; \xi_i)} \nonumber \\
        & \qquad \overset{\phantom{(d)}}{\leq}
        \frac{L}{n}\sum_{i=1}^n \bE\norm{\hat x_S^* - \hat x_{S_i'}^*},
        \label{eq-pf:(lm:gen-cc)-(xs)}
    \end{align}
    where $(a)$ holds by the inequality
    $\max_y h_1(y) - \max_y h_2(y) \leq \max_y \{h_1(y) - h_2(y)\}$ for any function $h_1$ and $h_2$, $(b)$ holds since $x_S^*$ and $x_{S_i'}^*$ have the same distribution, and $(c)$ follows from the definition of $\hat F_S(x,y)$ and the fact that $S_i'$ is independent from $\xi_i$ so one can first take expectation with respect to $\xi_i$. Similarly for $\hat y_S^*$, by the inequality
    $\min_x h_1(x) - \min_x h_2(x) \leq \max_x \{h_1(x) - h_2(x)\}$,
    \begin{align}
        &
        \min_{x\in\cX}\bE\squareBr*{\hat F_S(x, \hat y_S^*) + G(x, \hat y_S^*)} -
        \min_{x\in\cX}\bE\squareBr[\Big]{F(x, \hat y_S^*) + G(x, \hat y_S^*)}
        \nonumber \\
        & \qquad \leq
        \max_{x\in\cX} \curlyBr*{
        \bE[\hat F_S(x, \hat y_S^*)] - \bE[F(x, \hat y_S^*)]} \nonumber \\
        & \qquad =
        \max_{x\in\cX}\frac{1}{n}\sum_{i=1}^n\bE\squareBr*{
        f(x, \hat y_S^*; \xi_i)-f(x, \hat y_{S_i'}^*; \xi_i)} \nonumber \\
        & \qquad \leq
        \frac{L}{n}\sum_{i=1}^n \bE\norm{\hat y_S^* - \hat y_{S_i'}^*}.
        \label{eq-pf:(lm:gen-cc)-(ys)}
    \end{align}
    Combining the above two inequalities \eqref{eq-pf:(lm:gen-cc)-(xs)} and \eqref{eq-pf:(lm:gen-cc)-(ys)}, we obtain that
    \begin{align}
        \label{eq-pf:(lm:gen-cc)-(final)}
        &
        \max_{y\in\cY}\bE\squareBr[\Big]{F(\hat x_S^*, y) + G(\hat x_S^*, y)} -
        \min_{x\in\cX}\bE\squareBr[\Big]{F(x, \hat y_S^*) + G(x, \hat y_S^*)}
        \nonumber \\
        & \qquad \leq
        \max_{y\in\cY}\bE\squareBr*{
        \hat F_S(\hat x_S^*, y) + G(\hat x_S^*, y)} -
        \min_{x\in\cX}\bE\squareBr*{
        \hat F_S(x, \hat y_S^*) + G(x, \hat y_S^*)} \\
        & \qquad \qquad +
        \frac{L}{n}\sum_{i=1}^n \bE\squareBr*{
        \norm{\hat x_S^* - \hat x_{S_i'}^*} +
        \norm{\hat y_S^* - \hat y_{S_i'}^*}}. \nonumber
    \end{align}
    The first two terms in the RHS of \eqref{eq-pf:(lm:gen-cc)-(final)} can be bounded by Cauchy–Schwarz inequality and the optimality condition of the saddle point $(\hat x_S^*, \hat y_S^*)$,
    \begin{align*}
        &
        \max_{y\in\cY}\bE\squareBr*{\hat F_S(\hat x_S^*, y) + G(\hat x_S^*, y)} -
        \min_{x\in\cX}\bE\squareBr*{\hat F_S(x, \hat y_S^*) + G(x, \hat y_S^*)} \\
        & \qquad \leq \bE\squareBr*{
        \max_{y\in\cY}\curlyBr*{\hat F_S(\hat x_S^*, y) + G(\hat x_S^*, y)} -
        \min_{x\in\cX}\curlyBr*{\hat F_S(x, \hat y_S^*) + G(x, \hat y_S^*)}} \\
        & \qquad = 0.
    \end{align*}
    For the last term in the RHS of \eqref{eq-pf:(lm:gen-cc)-(final)}, we have that
    \begin{align}
        \roundBr*{\norm{\hat x_S^* - \hat x_{S_i'}^*} +
        \norm{\hat y_S^* - \hat y_{S_i'}^*}}^2
        & \leq
        2\norm{\hat x_S^* - \hat x_{S_i'}^*}^2 +
        2\norm{\hat y_S^* - \hat y_{S_i'}^*}^2 \nonumber \\
        & \leq
        \frac{2}{\mu}\roundBr*{\mu_x\norm{\hat x_S^* - \hat x_{S_i'}^*}^2 +
        \mu_y\norm{\hat y_S^* - \hat y_{S_i'}^*}^2} \nonumber \\
        & \leq
        \frac{8L^2}{\mu^2 n^2},
        \label{eq-pf:(lm:gen-cc)-(b)}
    \end{align}
    where the last inequality follows from the stability \eqref{eq-pf:(lm:gen-cc)-(stability)}. Then plugging the above two bounds back into \eqref{eq-pf:(lm:gen-cc)-(final)}, we obtain the generalization error of $(\hat x_S^*, \hat y_S^*)$:
    \begin{equation*}
        \max_{y\in\cY}\bE\squareBr[\Big]{F(\hat x_S^*, y) + G(\hat x_S^*, y)} -
        \min_{x\in\cX}\bE\squareBr[\Big]{F(x, \hat y_S^*) + G(x, \hat y_S^*)}
        \leq \frac{2\sqrt{2}L^2}{\mu n},
    \end{equation*}
    measured by the population weak duality gap.
\end{proof}

Based on the above results, we give the proof of Corollary \ref{cor:gen-cc} below.

\begin{proof}[Proof of Corollary \ref{cor:gen-cc}]
    When $v_S$ has dependence on $S$, by the same reason as \eqref{eq-pf:(lm:gen-cc)-(xs)}, we have
    \begin{align*}
        &
        \bE\squareBr[\Big]{F(\hat x_S^*, v_S) + G(\hat x_S^*, v_S)} -
        \bE\squareBr*{\hat F_S(\hat x_S^*, v_S) + G(\hat x_S^*, v_S)} \\
        & \qquad =
        \frac{1}{n} \sum_{i=1}^n \bE[F(\hat x_{S'_i}^*, v_{S'_i})] -
        \bE[\hat F_S(\hat x_S^*, v_S)] \\
        & \qquad =
        \frac{1}{n}\sum_{i=1}^n \bE\squareBr*{
        f(\hat x_{S'_i}^*, v_{S'_i}; \xi_i) -
        f(\hat x_S^*, v_S; \xi_i)} \\
        & \qquad \leq
        \frac{L}{n}\sum_{i=1}^n \bE\squareBr*{
        \norm{\hat x_{S'_i}^* - \hat x_S^*} + \norm{v_{S'_i} - v_S}},
    \end{align*}
    and similarly by \eqref{eq-pf:(lm:gen-cc)-(ys)}, we get
    \begin{equation*}
        \bE\squareBr*{\hat F_S(u_S, \hat y_S^*) + G(u_S, \hat y_S^*)} -
        \bE\squareBr[\Big]{F(u_S, \hat y_S^*) + G(u_S, \hat y_S^*)} \leq
        \frac{L}{n}\sum_{i=1}^n \bE\squareBr*{
        \norm{\hat y_{S'_i}^* - \hat y_S^*} + \norm{u_{S_i'} - u_S}}.
    \end{equation*}
    Moreover, since $(\hat x_S^*, \hat y_S^*)$ is the saddle point of
    $\hat F_S(x,y) + G(x,y)$, we have
    \begin{align*}
        &
        \hat F_S(\hat x_S^*, v_S) + G(\hat x_S^*, v_S) -
        \hat F_S(u_S, \hat y_S^*) - G(u_S, \hat y_S^*) \\
        & \qquad \leq
        \hat F_S(\hat x_S^*, \hat y_S^*) + G(\hat x_S^*, \hat y_S^*) -
        \hat F_S(\hat x_S^*, \hat y_S^*) - G(\hat x_S^*, \hat y_S^*) \\
        & \qquad = 0.
    \end{align*}
    As a result, by the three inequalities above, we have
    \begin{align*}
        &
        \bE\squareBr[\Big]{F(\hat x_S^*, v_S) + G(\hat x_S^*, v_S)} -
        \bE\squareBr[\Big]{F(u_S, \hat y_S^*) + G(u_S, \hat y_S^*)} \\
        & \qquad \leq
        \bE\squareBr*{\hat F_S(\hat x_S^*, v_S) + G(\hat x_S^*, v_S)} - \bE\squareBr*{\hat F_S(u_S, \hat y_S^*) + G(u_S, \hat y_S^*)} \\
        & \qquad \qquad +
        \frac{L}{n}\sum_{i=1}^n \bE\squareBr[\Big]{
        \norm{\hat x_S^* - \hat x_{S_i'}^*} +
        \norm{\hat y_S^* - \hat y_{S_i'}^*}} +
        \frac{L}{n}\sum_{i=1}^n \bE\squareBr[\Big]{
        \norm{u_S - u_{S_i'}} + \norm{v_S - v_{S_i'}}} \\
        & \qquad \leq
        \frac{2\sqrt{2}L^2}{\mu n} + L(\Delta_u + \Delta_v),
    \end{align*}
    where the last inequality follows from \eqref{eq-pf:(lm:gen-cc)-(b)} and stability of $u_S$, $v_S$.
\end{proof}

\subsection{Near-Linear Time Algorithms for Smooth Strongly-Convex--Strongly-Concave Functions} \label{sec:app-scsc}

Algorithm \ref{algo:perturb-scsc} achieves near-optimal utility guarantees on the strong duality gap for SC-SC functions with near-linear time-complexity. Its analysis is provided in the proof of Theorem \ref{thm:scsc} below.

\begin{proof}[Proof of Theorem \ref{thm:scsc}]
    For the privacy guarantee, we first bound the sensitivity of algorithm $\cA$. Given neighboring datasets $S\sim S'$, by Lemma \ref{lm:gen-scsc}, for $\mu=\min\{\mu_x, \mu_y\}$, we know that
    \begin{align*}
        \norm{\hat x_S^* - \hat x_{S'}^*}
        & \leq 
        \sqrt{\frac{1}{\mu_x}\roundBr*{
        \mu_x\norm{\hat x_S^* - \hat x_{S'}^*}^2 +
        \mu_y\norm{\hat y_S^* - \hat y_{S'}^*}^2}} \\
        & \leq
        \frac{2L}{n}\sqrt{\frac{1}{\mu_x\mu}}.
    \end{align*}
    For the same reason, by the guarantee of $\cA$ in Algorithm \ref{algo:perturb-scsc}, we have
    $\norm{A_x(S) - \hat x_S^*} \leq (L/n)\sqrt{1/(\mu_x\mu)}$ with probability at least $1-\delta/8$. Thus the sensitivity of $\cA_x$ is bounded as
    \begin{align}
        \norm{A_x(S) - A_x(S')}
        & \leq
        \norm{A_x(S) - \hat x_S^*} +
        \norm{\hat x_S^* - \hat x_{S'}^*} +
        \norm{\hat x_{S'}^* - A_x(S')} \nonumber \\
        & \leq
        \frac{4L}{n}\sqrt{\frac{1}{\mu_x\mu}},
        \label{eq-pf:(thm:scsc)-(sensitivity)}
    \end{align}
    with probability at least $1-\delta/4$ by the union bound. Applying Gaussian mechanism in Definition \ref{def:gauss}, $\cA_x+\cN(0,\sigma_x^2\rI_{d_x})$ is $(\varepsilon/2, \delta/2)$-DP when setting $\sigma_x=(8L/(n\varepsilon))\sqrt{2\log(5/\delta)/(\mu_x\mu)}$. Similarly we can bound the sensitivity of $\cA_y$ and the setting of $\sigma_y$ in Algorithm \ref{algo:perturb-scsc} guarantees $(\varepsilon/2, \delta/2)$-DP of $\cA_y+\cN(0,\sigma_y^2\rI_{d_y})$. Finally by the basic composition in Lemma \ref{lm:composition}, Algorithm \ref{algo:perturb-scsc} is $(\varepsilon, \delta)$-DP.
    
    Next, we prove the utility guarantees for the output $(\tilde x, \tilde y)$. For the empirical bound, we apply Lemma \ref{lm:gap-minimax} for $\hat F_S(x, y)$ under the assumption that $(\hat x_S^*, \hat y_S^*)$ is the interior point and get
    \begin{align}
        \label{eq-pf:(thm:scsc)-(init)}
        \bE\squareBr*{\max_{y\in\cY} \hat F_S(\tilde x, y) -
        \min_{x\in\cX} \hat F_S(x, \tilde y)}
        & \leq
        \frac{(\kappa_y+1)\ell}{2}\bE\norm{\tilde x - \hat x_S^*}^2 +
        \frac{(\kappa_x+1)\ell}{2}\bE\norm{\tilde y - \hat y_S^*}^2 \nonumber \\
        & \leq
        (\kappa_y+1)\ell\cdot\bE\squareBr*{
        \norm{\tilde x - \cA_x(S)}^2 +
        \norm{\cA_x(S) - \hat x_S^*}^2} \\
        & \qquad +
        (\kappa_x+1)\ell\cdot\bE\squareBr*{
        \norm{\tilde y - \cA_y(S)}^2 +
        \norm{\cA_y(S) - \hat y_S^*}^2}. \nonumber
    \end{align}
    According to the values of $\sigma_x$ and $\sigma_y$ in Algorithm \ref{algo:perturb-scsc} to guarantee $(\varepsilon, \delta)$-DP, we know
    \begin{align}
        &
        (\kappa_y+1)\ell\cdot\bE\norm{\tilde x - \cA_x(S)}^2 +
        (\kappa_x+1)\ell\cdot\bE\norm{\tilde y - \cA_y(S)}^2 \nonumber \\
        & \qquad \leq
        (\kappa_x\kappa_y + \kappa)\cdot\bE\squareBr*{
        \mu_x\norm{\tilde x - \cA_x(S)}^2 +
        \mu_y\norm{\tilde y - \cA_y(S)}^2} \nonumber \\
        & \qquad \leq
        256L^2(\kappa_x\kappa_y + \kappa)
        \frac{d\log(5/\delta)}{\mu n^2\varepsilon^2},
        \label{eq-pf:(thm:scsc)-(noise)}
    \end{align}
    where we let $\mu=\min\{\mu_x, \mu_y\}$, $\kappa=\ell/\mu$ and $d=\max\{d_x, d_y\}$. Similarly by the guarantee of the algorithm $\cA$ in Remark \ref{rmk:scsc}, we have
    \begin{align}
        &
        (\kappa_y+1)\ell\cdot\bE\norm{\cA_x(S) - \hat x_S^*}^2 +
        (\kappa_x+1)\ell\cdot\bE\norm{\cA_y(S) - \hat y_S^*}^2 \nonumber \\
        & \qquad \leq 
        (\kappa_x\kappa_y + \kappa)\cdot\bE\squareBr*{
        \mu_x\norm{\cA_x(S) - \hat x_S^*}^2 +
        \mu_y\norm{\cA_y(S) - \hat y_S^*}^2} \nonumber \\
        & \qquad \leq
        L^2(\kappa_x\kappa_y + \kappa)\frac{\delta}{8\mu n^2}.
        \label{eq-pf:(thm:scsc)-(acc)}
    \end{align}
    As a result of above two inequalities and \eqref{eq-pf:(thm:scsc)-(init)}, when $d\geq 1$, $\varepsilon<1$ and $\delta<1/n$, we have that
    \begin{equation*}
        \bE\squareBr*{\max_{y\in\cY} \hat F_S(\tilde x, y) -
        \min_{x\in\cX} \hat F_S(x, \tilde y)}
        \leq 257L^2(\kappa_x\kappa_y + \kappa)
        \frac{d\log(5/\delta)}{\mu n^2\varepsilon^2}.
    \end{equation*}
    
    Similarly, we can apply Lemma \ref{lm:gap-minimax} for $F(x,y)$ to obtain the population guarantee as
    \begin{align*}
        \bE\squareBr*{\max_{y\in\cY} F(\tilde x, y) -
        \min_{x\in\cX} F(x, \tilde y)}
        & \leq
        \frac{(\kappa_y+1)\ell}{2}\bE\norm{\tilde x - x^*}^2 +
        \frac{(\kappa_x+1)\ell}{2}\bE\norm{\tilde y - y^*}^2) \\
        & \leq
        \frac{3(\kappa_y+1)\ell}{2}\bE\squareBr*{
        \norm{\tilde x - \cA_x(S)}^2 + \norm{\cA_x(S) - \hat x_S^*}^2 +
        \norm{\hat x_S^* - x^*}^2} \\
        & \qquad +
        \frac{3(\kappa_x+1)\ell}{2}\bE\squareBr*{
        \norm{\tilde y - \cA_y(S)}^2 + \norm{\cA_y(S) - \hat y_S^*}^2 +
        \norm{\hat y_S^* - y^*}^2} \\
        & \leq
        L^2(\kappa_x\kappa_y + \kappa)\roundBr*{
        \frac{384d\log(5/\delta)}{\mu n^2\varepsilon^2} + 
        \frac{3\delta}{16\mu n^2} + \frac{6\sqrt{2}}{\mu n}} \\
        & <
        3L^2(\kappa_x\kappa_y + \kappa)\roundBr*{
        \frac{3}{\mu n} +
        \frac{128d\log(5/\delta)}{\mu n^2\varepsilon^2}},
    \end{align*}
    where we use \eqref{eq-pf:(thm:scsc)-(noise)} and \eqref{eq-pf:(thm:scsc)-(acc)}, and the bound for the distance between $(\hat x_S^*, \hat y_S^*)$ and $(x^*, y^*)$ follows from Lemma \ref{lm:gap-minimax} and the generalization results in Lemma \ref{lm:gen-scsc} such that
    \begin{equation*}
        \bE\squareBr*{\frac{\mu_x}{2}\norm{\hat x_S^*-x^*}^2 +
        \frac{\mu_y}{2}\norm{\hat y_S^*-y^*}^2} \leq
        \frac{2\sqrt{2}L^2}{\mu n},
    \end{equation*}
    where $\mu=\min\{\mu_x,\mu_y\}$ simplifies the notations. We conclude the proof by a remark. When the saddle points are not interior points, we can instead obtain the guarantee on the distance between $(\tilde x, \tilde y)$ and $(\hat x_S^*, \hat y_S^*)$ or $(x^*, y^*)$, which is weaker than the duality gap.
\end{proof}

\subsection{Near-Linear Time Algorithms for Smooth Convex-Concave Functions} \label{sec:app-cc}

In this section, we analyze the algorithm for smooth convex-concave DP-SMO. We first present a lemma that extends the results that the proximal operator is non-expansive to minimax settings.

\begin{lemma}
    Let $f(x,y)$ be a convex-concave function on the closed convex domain $\cX\times\cY$. For some $(u,v)\in\cX\times\cY$ and $\mu_x, \mu_y>0$, we define
    \begin{equation*}
        F_{u,v}(x,y) \triangleq f(x,y) +
        \frac{\mu_x}{2}\norm{x-u}^2 - \frac{\mu_y}{2}\norm{y-v}^2.
    \end{equation*}
    Denote the saddle point of $F_{u,v}(x,y)$ as
    $(x_{u,v}^*, y_{u,v}^*)\in\cX\times\cY$. Then it holds that
    \begin{equation*}
        \mu_x\norm{x_{u,v}^* - x_{u',v'}^*}^2 +
        \mu_y\norm{y_{u,v}^* - y_{u',v'}^*}^2 \leq
        \mu_x\norm{u-u'}^2 + \mu_y\norm{v-v'}^2,
    \end{equation*}
    where $(x_{u',v'}^*, y_{u',v'}^*)\in\cX\times\cY$ is the saddle point of $F_{u', v'}(x,y)$ defined in the same way as $F_{u,v}(x,y)$ for some point
    $(u', v')\in\cX\times\cY$.
    \label{lm:prox-minimax}
\end{lemma}

\begin{proof}
    For $(x,y)\in\bR^{d_x}\times\bR^{d_y}$, we define
    $f_{\cX,\cY}(x,y) := f(x,y) + \cI_{\cX}(x) - \cI_{\cY}(y)$, where $\cI_{C}$ is the indicator function of some set $C$, i.e., $\cI_C(x)=0$ if $x\in C$ and $\cI_C(x)=\infty$ otherwise. Since $\cI_{\cX}$ and $\cI_{\cY}$ are both convex when the domain is convex, $f_{\cX, \cY}(x,y)$ is convex-concave. By the above definition, $(x_{u,v}^*, y_{u,v}^*)$ is the saddle point of
    $\bar F_{u,v}(x,y) := f_{\cX,\cY}(x,y) + (\mu_x/2)\norm{x-u}^2 - (\mu_y/2)\norm{y-v}^2$. Applying the optimality condition, i.e.,
    $0\in\partial\bar F_{u,v}(x_{u,v}^*, y_{u,v}^*)$,
    we know that
    \begin{equation}
        \mu_x (u - x_{u,v}^*) \in \partial_x f_{\cX,\cY}(x_{u,v}^*, y_{u,v}^*),
        \quad
        \mu_y (v - y_{u,v}^*) \in -\partial_y f_{\cX,\cY}(x_{u,v}^*, y_{u,v}^*),
        \label{eq-pf:(lm:prox-minimax)-(uv)}
    \end{equation}
    where $\partial_x f_{\cX,\cY}$ and $\partial_y f_{\cX,\cY}$ are partial subgradients. Similarly for $(x_{u',v'}^*, y_{u',v'}^*)$ and $\bar F_{u',v'}$, we have
    \begin{equation}
        \mu_x (u' - x_{u',v'}^*) \in
        \partial_x f_{\cX,\cY}(x_{u',v'}^*, y_{u',v'}^*), \quad
        \mu_y (v' - y_{u',v'}^*) \in
        -\partial_y f_{\cX,\cY}(x_{u',v'}^*, y_{u',v'}^*).
        \label{eq-pf:(lm:prox-minimax)-(u'v')}
    \end{equation}
    By the property that $(\partial_x f_{\cX,\cY}, -\partial_y f_{\cX,\cY})$ is a monotone operator when $f_{\cX,\cY}(x,y)$ is convex-concave (see \citep[Theorem 1]{rockafellar1970monotone} or \citep[Lemma 1]{farnia2021train}), it holds for any $x_1, x_2\in\bR^{d_x}$,
    $y_1, y_2\in\bR^{d_y}$, and subgradients
    $(g_x, -g_y)\in(\partial_x f_{\cX,\cY}, -\partial_y f_{\cX,\cY})$ that
    \begin{align*}
        \roundBr[\big]{g_x(x_1, y_1) - g_x(x_2, y_2)}^\top
        \roundBr[\big]{x_1 - x_2} -
        \roundBr[\big]{g_y(x_1, y_1) - g_y(x_2, y_2)}^\top
        \roundBr[\big]{y_1 - y_2} \geq 0.
    \end{align*}
    Therefore, using the set membership relations \eqref{eq-pf:(lm:prox-minimax)-(uv)} and \eqref{eq-pf:(lm:prox-minimax)-(u'v')}, we get that
    \begin{equation*}
        \mu_x \roundBr[\Big]{(u - u') - (x_{u,v}^* - x_{u',v'}^*)}^\top
        \roundBr[\Big]{x_{u,v}^* - x_{u',v'}^*} +
        \mu_y \roundBr[\Big]{(v - v') - (y_{u,v}^* - y_{u',v'}^*)}^\top
        \roundBr[\Big]{y_{u,v}^* - y_{u',v'}^*} \geq 0.
    \end{equation*}
    Hence, we obtain that
    \begin{align*}
        &
        \mu_x\norm{x_{u,v}^* - x_{u',v'}^*}^2 +
        \mu_y \norm{y_{u,v}^* - y_{u',v'}^*}^2 \\
        & \qquad \leq
        \mu_x (u - u')^\top (x_{u,v}^* - x_{u',v'}^*) +
        \mu_y (v - v')^\top (y_{u,v}^* - y_{u',v'}^*) \\
        & \qquad \leq
        \mu_x \norm{u - u'} \norm{x_{u,v}^* - x_{u',v'}^*} +
        \mu_y \norm{v - v'} \norm{y_{u,v}^* - y_{u',v'}^*} \\
        & \qquad \leq
        \sqrt{\mu_x\norm{x_{u,v}^* - x_{u',v'}^*}^2 +
        \mu_y \norm{y_{u,v}^* - y_{u',v'}^*}^2} \cdot
        \sqrt{\mu_x\norm{u - u'}^2 + \mu_y \norm{v-v'}^2},
    \end{align*}
    by Cauchy-Schwarz inequality. The proof is thus complete.
\end{proof}

The proof of Lemma \ref{lm:gen-algo_cc} directly follows from Corollary \ref{cor:gen-cc} and above Lemma \ref{lm:prox-minimax}.

\begin{proof}[Proof of Lemma \ref{lm:gen-algo_cc}]
    Applying Corollary \ref{cor:gen-cc} for $\hat F_k(x,y)$ with regularization term $(\mu_k/2)\norm{x-\tilde x_{k-1}}^2 - (\mu/2)\norm{y}^2$ and dataset $S_k:=\{\xi_i\}_{i=(k-1)\bar n+1}^{k\bar n}$, we have that for any $x\in\cX$ and $y\in\cY$,
    \begin{align}
        \bE[F(\hat x_k^*, y) - F(x, \hat y_k^*)]
        & \leq
        \frac{\mu_k}{2}\bE\norm{x-\tilde x_{k-1}}^2 -
        \frac{\mu_k}{2}\bE\norm{\hat x_k^*-\tilde x_{k-1}}^2 +
        \frac{\mu}{2}\bE\norm{y}^2 - \frac{\mu}{2}\bE\norm{\hat y_k^*}^2
        \nonumber \\
        & \qquad +
        \frac{2\sqrt{2}L^2}{\mu\bar n} + L(\Delta_x + \Delta_y),
        \label{eq-pf:(lm:gen-algo_cc)-(a)}
    \end{align}
    where $\Delta_x$ and $\Delta_y$ are the stability bounds of $x$ and $y$ with respect to $S_k$. Let $x=\hat x_{k-1}^*$ and $y=\hat y_{k+1}^*$ in the above inequality. Note that $\hat x_{k-1}^*$ is independent of $S_k$ and then $\Delta_x=0$ for $1\leq k \leq K$. For $\Delta_y$, we need to compute the difference of $\hat y_{k+1}^*$ given neighboring datasets $S_k\sim S_{k'}$. In the following analysis, we denote the saddle point as $(\hat x_{k'}^*, \hat y_{k'}^*)$, the output of $\cA$ as $(x_{k'}, y_{k'})$ and the perturbed output as
    $\tilde x_{k'}$ corresponding to the dataset $S_{k'}$. When $k\leq K-1$, since $\hat y_{k+1}^*$ is the saddle point of
    \begin{equation*}
        \frac{1}{\bar n}\sum_{i\in S_{k+1}} f(x,y;\xi_i) + \frac{\mu_{k+1}}{2}\norm{x-\tilde x_k}^2 - \frac{\mu}{2}\norm{y}^2,
    \end{equation*}
    and $\hat y_{k'+1}^*$ is the saddle point of
    \begin{equation*}
        \frac{1}{\bar n}\sum_{i\in S_{k+1}} f(x,y;\xi_i) + \frac{\mu_{k+1}}{2}\norm{x-\tilde x_{k'}}^2 - \frac{\mu}{2}\norm{y}^2,
    \end{equation*}
    we can conclude from Lemma \ref{lm:prox-minimax} that
    \begin{align*}
        \mu\bE\norm{\hat y_{k+1}^* - \hat y_{k'+1}^*}^2
        & \leq
        \mu_{k+1}\bE\norm{\tilde x_k - \tilde x_{k'}}^2 \\
        & =
        \mu_{k+1}\bE\norm{x_k - x_{k'}}^2 \\
        & \leq
        3\mu_{k+1}\roundBr*{\bE\norm{x_k - \hat x_k^*}^2 +
        \bE\norm{\hat x_k^* - \hat x_{k'}^*}^2 +
        \bE\norm{x_{k'} - \hat x_{k'}^*}^2} \\
        & \leq
        \frac{\mu_{k+1}}{\mu_k}\roundBr*{\frac{12L^2}{\mu \bar n^2} +
        \frac{3\delta L^2}{4\mu\bar n^2}} \\
        & \leq
        \frac{26L^2}{\mu\bar n^2},
    \end{align*}
    where the equality holds since the only difference is due to the neighboring datasets $S_k\sim S_{k'}$ and the third inequality follows from Lemma \ref{lm:gen-cc} and guarantees of $\cA$ in Remark \ref{rmk:scsc}. Therefore we obtain that $\Delta_y=\sqrt{26}L/(\mu\bar n)$ for $\hat y_{k+1}^*$ when
    $k\leq K-1$. When $k=K$, by the definition that
    $\hat y_{K+1}^*\in\arg\max_{y\in\cY}\bE[F(\tilde x_K, y)]$, we know that $\Delta_y=0$ since $\hat y_{K+1}^*$ is independent of $S$. Then by \eqref{eq-pf:(lm:gen-algo_cc)-(a)}, it holds for all $k=1,\cdots,K$ that
    \begin{equation*}
        \bE[F(\hat x_k^*, \hat y_{k+1}^*) - F(\hat x_{k-1}^*, \hat y_k^*)]
        \leq
        \frac{\mu_k}{2}\bE\norm{\hat x_{k-1}^* - \tilde x_{k-1}}^2 +
        \frac{\mu}{2}\bE\squareBr*{
        \norm{\hat y_{k+1}^*}^2-\norm{\hat y_k^*}^2} +
        \frac{8L^2}{\mu\bar n}
    \end{equation*}
    since $\norm{\hat x_k^* - \tilde x_{k-1}}^2\geq 0$ and $2\sqrt{2}+\sqrt{26}<8$.
\end{proof}

With Lemma \ref{lm:gen-algo_cc}, we are ready to give the proof of Theorem \ref{thm:cc}.

\begin{proof}[Proof of Theorem \ref{thm:cc}]
    The population function $F(x,y)$ has at least one saddle point on the domain $\cX\times\cY$, and we denote it as $(x^*, y^*)$. First, we show that Algorithm \ref{algo:perturb-phased_cc} is $(\varepsilon/2, \delta/2)$-DP and give the utility bound of its output $\tilde x_K$. Similar to the proof of Theorem \ref{thm:scsc}, we can obtain guarantees of each phase in Algorithm \ref{algo:perturb-phased_cc}. By Lemma \ref{lm:gen-cc}, since $\min\{\mu_k, \mu\}=\mu$, we know that the empirical solution $\hat x_k^*$ has stability $2L/(\bar n\sqrt{\mu_k\mu})$. With the guarantee of the algorithm $\cA$ and the same statement as \eqref{eq-pf:(thm:scsc)-(sensitivity)}, the sensitivity of $x_k$ is bounded by $4L/(\bar n\sqrt{\mu_k\mu})$ with probability $1-\delta/4$, and thus the values of $\sigma_{k}$ guarantee $(\varepsilon/2, \delta/2)$-DP. As a result, by \eqref{eq-pf:(thm:scsc)-(noise)} and \eqref{eq-pf:(thm:scsc)-(acc)},
    \begin{align}
        \bE\norm{\tilde x_k - \hat x_k^*}^2
        & \leq
        2\bE\norm{\tilde x_k-x_k}^2 + 2\bE\norm{x_k-\hat x_k^*}^2 \nonumber \\
        & \leq
        2d_x\sigma_k^2 + \frac{L^2}{4\mu_k}\frac{\delta}{\mu\bar n^2} \nonumber \\
        & \leq
        \frac{257L^2}{\mu_k}
        \frac{d_x\log(5/\delta)}{\mu\bar n^2\varepsilon^2}.
        \label{eq-pf:(thm:cc-primal)-(phase)}
    \end{align}
    
    Then we analyze the full algorithm. By the parallel composition in Lemma \ref{lm:parallel}, Algorithm \ref{algo:perturb-phased_cc} is $(\varepsilon/2, \delta/2)$-DP since we use disjoint datasets for different phases and each phase is $(\varepsilon/2, \delta/2)$-DP. For the utility bound of the output $\tilde x_K$, we start with the following decomposition:
    \begin{equation}
        \max_{y\in\cY}\bE[F(\tilde x_K, y)] - \bE[F(x^*, \hat y_1^*)] =
        \bE[F(\tilde x_K, \hat y_{K+1}^*) - F(\hat x_K^*, \hat y_{K+1}^*)] +
        \sum_{k=1}^K \bE \squareBr[\Big]{
        F(\hat x_k^*, \hat y_{k+1}^*) - F(\hat x_{k-1}^*, \hat y_k^*)},
        \label{eq-pf:(thm:cc-primal)-(error-decomp)}
    \end{equation}
    where we let $\hat x_0^*=x^*$ for the saddle point of $F(x,y)$ and
    $\hat y_{K+1}^*\in\arg\max_{y\in\cY}\bE[F(\tilde x_K, y)]$ to simplify the analysis. The first term in the RHS of \eqref{eq-pf:(thm:cc-primal)-(error-decomp)} can be bounded by Lipschitzness of $F(x,y)$:
    \begin{align}
        \bE[F(\tilde x_K, \hat y_{K+1}^*) - F(\hat x_K^*, \hat y_{K+1}^*)]
        & \leq
        L\sqrt{\bE\norm{\tilde x_K- \hat x_K^*}^2} \nonumber \\
        & <
        \frac{17L^2}{\sqrt{\mu_K\mu}}
        \frac{\sqrt{d_x\log(5/\delta)}}{\bar n\varepsilon} \nonumber \\
        & =
        \frac{17L^2}{\mu\sqrt{n}}
        \frac{\sqrt{d_x\log(5/\delta)}}{\bar n\varepsilon} \nonumber \\
        & \leq
        17LD\cdot\frac{\sqrt{d\log(5/\delta)}}{2\bar n\varepsilon},
        \label{eq-pf:(thm:cc-primal)-(lips)}
    \end{align}
    where the second inequality uses the guarantee of phase $K$ in \eqref{eq-pf:(thm:cc-primal)-(phase)}, the equality is due to the setting that $\mu_K=\mu n$ and the last inequality follows from the choice that $\mu\geq 2L/(D\sqrt{n})$ and $d_x\leq d$. Therefore, with Lemma \ref{lm:gen-algo_cc} to bound the second term in the RHS of \eqref{eq-pf:(thm:cc-primal)-(error-decomp)}, we obtain that
    \begin{align}
        \max_{y\in\cY}\bE[F(\tilde x_K, y)] - \bE[F(x^*, \hat y_1^*)]
        & \leq
        17LD\cdot\frac{\sqrt{d\log(5/\delta)}}{2\bar n\varepsilon} +
        \sum_{k=1}^K \roundBr*{
        \frac{\mu_k}{2}\bE\norm{\hat x_{k-1}^* - \tilde x_{k-1}}^2 + \frac{8L^2}{\mu\bar n}} \nonumber \\
        & \qquad +
        \frac{\mu}{2}\sum_{k=1}^K
        \bE\squareBr*{\norm{\hat y_{k+1}^*}^2-\norm{\hat y_k^*}^2}
        \nonumber \\
        & \leq
        17LD\cdot\frac{\sqrt{d\log(5/\delta)}}{2\bar n\varepsilon} +
        \sum_{k=2}^K\frac{257\mu_k}{2\mu_{k-1}}
        \frac{L^2d_x\log(5/\delta)}{\mu\bar n^2\varepsilon^2} \nonumber \\
        & \qquad +
        \sum_{k=1}^K\frac{8L^2}{\mu\bar n} + \frac{\mu_1}{2}\norm{x^* - x_0}^2 + \frac{\mu}{2}\norm{\hat y_{K+1}^*}^2 \nonumber \\
        & \leq
        4LDK^2\roundBr*{\frac{1}{\sqrt{n}} + \frac{5\sqrt{d\log(5/\delta)}}{n\varepsilon}} +
        \frac{\mu}{2}(2\norm{x^* - x_0}^2 + \norm{\hat y_{K+1}^*}^2)
        \nonumber \\
        & \leq
        8LDK^2\roundBr*{\frac{1}{\sqrt{n}} + \frac{5\sqrt{d\log(5/\delta)}}{n\varepsilon}},
        \label{eq-pf:(thm:cc-primal)-(final)}
    \end{align}
    where the second inequality is due to the guarantees of $\tilde x_{k-1}$ in \eqref{eq-pf:(thm:cc-primal)-(phase)} for $k\geq 2$ and the settings that $\tilde x_0=x_0$, $\hat x_0^*=x^*$, the third inequality holds by the choice that $\mu_k=\mu\cdot 2^k$, $\mu=(L/D)\max\{2/\sqrt{n}, 13\log(n)\sqrt{d\log(5/\delta)}/(n\varepsilon)\}$, and the last inequality follows from the assumptions that $\norm{x^*}^2\leq D^2$ and
    $\norm{\hat y_{K+1}^*}^2\leq D^2$ when the initialization is $x_0=0$. Since $(x^*, y^*)$ is the saddle point of $F(x,y)$, we know that $F(x^*, \hat y_1^*) \leq F(x^*, y^*)$, and thus
    \begin{equation}
        \max_{y\in\cY}\bE[F(\tilde x_K, y)] \leq F(x^*, y^*) + 8LDK^2\roundBr*{\frac{1}{\sqrt{n}} + \frac{5\sqrt{d\log(5/\delta)}}{n\varepsilon}}.
        \label{eq-pf:(thm:cc)-(x)}
    \end{equation}

    Here we give the privacy and utility guarantees of the primal solution $\tilde x_K$. The dual solution comes from a symmetric Algorithm \ref{algo:perturb-phased_cc-II}. The same as the above analysis for Algorithm \ref{algo:perturb-phased_cc}, we can show that Algorithm \ref{algo:perturb-phased_cc-II} is $(\varepsilon/2, \delta/2)$-DP and give the utility bound of its output $\tilde y_K$. Without causing confusion, we borrow notations from Algorithm \ref{algo:perturb-phased_cc} for simplicity. Since everything is very much similar by switching the role of the primal $x$ and the dual $y$, we will not repeat all the details.
    
    First, each phase of Algorithm \ref{algo:perturb-phased_cc-II} is $(\varepsilon/2, \delta/2)$-DP since with probability at least $1-\delta/4$, the sensitivity of $y_k$ is bounded by $4L/(\bar n\sqrt{\mu_k\mu})$. Similar to \eqref{eq-pf:(thm:cc-primal)-(phase)}, the output $\tilde y_k$ for $1\leq k\leq K$ satisfies that
    \begin{equation}
        \bE\norm{\tilde y_k - \hat y_k^*}^2 \leq \frac{257L^2}{\mu_k} \frac{d_y\log(5/\delta)}{\mu\bar n^2\varepsilon^2}.
        \label{eq-pf:(thm:cc)-(phase-y)}
    \end{equation}
    Then by the parallel composition of differential privacy, Algorithm \ref{algo:perturb-phased_cc-II} is $(\varepsilon/2, \delta/2)$-DP. For the utility bound of the output $\tilde y_K$, we have the following error decomposition that mirrors \eqref{eq-pf:(thm:cc-primal)-(error-decomp)}:
    \begin{equation}
        \bE[F(\hat x_1^*, y^*)] - \min_{x\in\cX}\bE[F(x, \tilde y_K)] =
        \sum_{k=1}^K \bE \squareBr[\Big]{
        F(\hat x_k^*, \hat y_{k-1}^*) - F(\hat x_{k+1}^*, \hat y_k^*)} +
        \bE[F(\hat x_{K+1}^*, \hat y_K^*) - F(\hat x_{K+1}^*, \tilde y_K)],
        \label{eq-pf:(thm:cc)-(error-decomp-y)}
    \end{equation}
    where $(\hat x_k^*, \hat y_k^*)$ is the saddle point of the regularized empirical function $\hat F_k(x,y)$ for $1\leq k\leq K$, and we let $\hat y_0^*=y^*$ for the saddle point of $F(x,y)$ and
    $\hat x_{K+1}^*\in\arg\min_{x\in\cX}\bE[F(x, \tilde y_K)]$ to simplify the analysis. The first term in the RHS of \eqref{eq-pf:(thm:cc)-(error-decomp-y)} can be bounded by a similar result as Lemma \ref{lm:gen-algo_cc}. By Corollary \ref{cor:gen-cc} and Lemma \ref{lm:prox-minimax}, setting $x=\hat x_{k+1}^*$ and $y=\hat y_{k-1}^*$, we obtain that
    \begin{equation*}
        \bE[F(\hat x_k^*, \hat y_{k-1}^*) - F(\hat x_{k+1}^*, \hat y_k^*)] \leq \frac{\mu_k}{2}\bE\norm{\hat y_{k-1}^* - \tilde y_{k-1}}^2 + \frac{8L^2}{\mu\bar n} + \frac{\mu}{2}\bE\squareBr*{
        \norm{\hat x_{k+1}^*}^2 -\norm{\hat x_{k}^*}^2}.
    \end{equation*}
    The second term in the RHS of \eqref{eq-pf:(thm:cc)-(error-decomp-y)} can be bounded by Lipschitzness of $F(x,y)$,
    \begin{equation*}
        \bE[F(\hat x_{K+1}^*, \hat y_K^*) - F(\hat x_{K+1}^*, \tilde y_K)]
        \leq 17LD\cdot\frac{\sqrt{d\log(5/\delta)}}{2\bar n\varepsilon},
    \end{equation*}
    which is the same as \eqref{eq-pf:(thm:cc-primal)-(lips)} using the guarantee of $\tilde y_K$ in \eqref{eq-pf:(thm:cc)-(phase-y)}. Finally plugging the above two bounds back into \eqref{eq-pf:(thm:cc)-(error-decomp-y)}, by the same reason as \eqref{eq-pf:(thm:cc-primal)-(final)}, we get that
    \begin{align*}
        \bE[F(\hat x_1^*, y^*)] - \min_{x\in\cX}\bE[F(x, \tilde y_K)]
        & \leq
        \sum_{k=2}^K\frac{257\mu_k}{2\mu_{k-1}}
        \frac{L^2d_y\log(5/\delta)}{\mu\bar n^2\varepsilon^2} + \sum_{k=1}^K\frac{8L^2}{\mu\bar n} \\
        & \qquad +
        \frac{\mu_1}{2}\norm{y^* - y_0}^2 +
        \frac{\mu}{2}\norm{\hat x_{K+1}^*}^2 +
        17LD\cdot\frac{\sqrt{d\log(5/\delta)}}{2\bar n\varepsilon} \\
        & \leq
        8LDK^2\roundBr*{\frac{1}{\sqrt{n}} + \frac{5\sqrt{d\log(5/\delta)}}{n\varepsilon}},
    \end{align*}
    where we use the guarantees of $\tilde y_k$ in \eqref{eq-pf:(thm:cc)-(phase-y)} and the assumption that $\norm{y^*}^2\leq D^2$ and $\norm{\hat x_{K+1}^*}^2\leq D^2$ when the initialization is $y_0=0$. Finally since $(x^*, y^*)$ is the saddle point of $F(x,y)$, we know that $F(\hat x_1^*, y^*) \geq F(x^*, y^*)$, and thus it holds that
    \begin{equation}
        - \min_{x\in\cX}\bE[F(x, \tilde y_K)] \leq - F(x^*, y^*) +
        8LDK^2\roundBr*{\frac{1}{\sqrt{n}} + \frac{5\sqrt{d\log(5/\delta)}}{n\varepsilon}}.
        \label{eq-pf:(thm:cc)-(y)}
    \end{equation}
    
    By basic composition in Lemma \ref{lm:composition}, the composition $(\tilde x_K, \tilde y_K)$ of Algorithm \ref{algo:perturb-phased_cc} and \ref{algo:perturb-phased_cc-II} is $(\varepsilon, \delta)$-DP. The proof is thus complete summing up \eqref{eq-pf:(thm:cc)-(x)} and \eqref{eq-pf:(thm:cc)-(y)}. Note that we only require that $(x^*, y^*)$ and $(\hat x_{K+1}^*, \hat y_{K+1}^*)$ have bounded norms, which is slightly weaker than assuming bounded domains.
\end{proof}

\end{document}